\documentclass[journal]{IEEEtran}
\ifCLASSINFOpdf
\else
\fi

\usepackage{times}
\usepackage{amsbsy,latexsym}
\usepackage[usenames]{color}
\usepackage{booktabs}
\usepackage{url}
\usepackage{color}
\usepackage{algorithm}
\usepackage{algorithmic}
\usepackage{multirow}
\usepackage{graphicx}
\usepackage{subfigure}
\usepackage{indentfirst}
\usepackage{amsmath} 
\usepackage{amssymb}  
\usepackage{amsthm}
\usepackage{amsfonts}
\usepackage{makecell}
\usepackage{indentfirst}
\usepackage{chngcntr}
\usepackage{hyperref}

\DeclareMathOperator*{\argmin}{arg\,min}

\newcommand{\myparatight}[1]{\smallskip\noindent{\bf {#1}:}~}

\newtheorem{theorem}{Theorem}

\newtheorem{problem}{Problem}

\begin{document}
%
\title{Robust Multi-subspace Analysis Using Novel Column $L_0$-norm Constrained Matrix Factorization}
%
%
%


\author{Binghui~Wang,~\IEEEmembership{Student~Member,~IEEE,}
       and~Chuang~Lin,~\IEEEmembership{Member,~IEEE}
\thanks{Binghui Wang is with the Electrical and Computer Engineering Department, Iowa State University, Ames, IA, 50010 USA, e-mail: binghuiw@iastate.edu.}
\thanks{Chuang Lin is with the Shenzhen Institute of Advanced Technology, Chinese Academy of Science, Shenzhen, China, e-mail: chuang.lin@siat.ac.cn.}
\thanks{Binghui Wang and Chuang Lin contributed equally to this work.}
\thanks{The preliminary work was done in Dalian University of Technology, China, and was further extended and completed in Iowa State University, USA.}
\thanks{Preliminary results of the proposed method have been published in~\cite{wang2015}.}} 

\maketitle

\begin{abstract}
We study the underlying structure of data (approximately) generated from a union of independent subspaces. Traditional methods learn only one subspace, failing to discover the multi-subspace structure, while state-of-the-art methods analyze the multi-subspace structure using data themselves as the dictionary, which cannot offer the explicit basis to span each subspace and are sensitive to errors via an indirect representation. Additionally, they also suffer from a high computational complexity, being quadratic or cubic to the sample size. 

To tackle all these problems, we propose a method, called Matrix Factorization with Column $L_0$-norm constraint (MFC$_0$), that can simultaneously \emph{learn the basis} for each subspace, \emph{generate direct sparse representation} for each data sample, as well as \emph{removing errors} in the data \emph{in an efficient way}. 
Furthermore, we develop a first-order alternating direction algorithm, whose computational complexity is linear to the sample size, to stably and effectively solve the nonconvex objective function and nonsmooth $l_0$-norm constraint of MFC$_0$. 
Experimental results on both synthetic and real-world datasets demonstrate that besides the superiority over traditional and state-of-the-art methods for subspace clustering, data reconstruction, error correction, MFC$_0$ also shows its uniqueness for multi-subspace basis learning and direct sparse representation.

\end{abstract}


%
\IEEEpeerreviewmaketitle

\section{Introduction}

\IEEEPARstart{T}{he} observed data are extremely high dimensional in this ``big data" era. Typically, these data reside in a much lower-dimensional latent subspace, instead of being uniformly distributed in the high-dimensional ambient space. Thus, it is of great importance to reveal the underlying structure of the data as it helps to reduce the computational cost and enables a compact representation for learning. Such an idea has been successfully applied to various reseach communities, e.g., dimensionality reduction~\cite{tenenbaum2000global, roweis2000nonlinear, belkin2003laplacian}, face recognition~\cite{belhumeur1997eigenfaces, he2005face, wang2014neighbourhood}, metric learning~\cite{weinberger2005distance}, etc..

Subspace methods have been widely used to analyze the data, and \emph{linear subspace}\footnote{A vector space is a subset of some other higher-dimensional vector space.} is the most common choice for its simplicity and computational efficiency. Additionally, linear subspace has shown its effectiveness in modeling real-world problems such as motion segmentation~\cite{rao2010motion,yan2006general}, face clustering~\cite{Liu:LLRR,basri2003lambertian}, and handwritten digits recognition~\cite{tangstructure}. Consequently, subspace analysis has been paid much attention in the past decades. For instance, principal component analysis (PCA) aims to learn a subspace while retaining maximal variances of the data. Nonnegative matrix factorization (NMF)~\cite{lee1999learning} is designed to learn both the nonnegative basis and nonnegative parts-based representation.
Robust PCA (RPCA)~\cite{Candes:RPCA} assumes that the data are approximately drawn from a low-rank subspace while perturbed by sparse noise. The basic assumption of these methods is the \emph{single} subspace, which is not the case in many practical applications. A more reasonable way is to consider the data as lying on or near a union of linear subspaces. Unfortunately, the generalization to handle multiple subspaces is quite challenging.

Recently, multi-subspace analysis has attracted increasingly interests in visual data analysis~\cite{liu2014learning,liu2010unsupervised,vidal2005generalized}. Derived from recent advances in compressive sensing~\cite{candes2009exact,donoho2006most}, the methods including~
\cite{SSC:PAMI2014,LRR:PAMI2014,vidal2014low,liu2012fixed,ni2010robust,Peng:2015:SCU,liu2017blessing} 
have incorporated sparse and/or low-rank regularization into their formulations to model the mixture of linear structures for clean data\footnote{Data points are strictly sampled from the respective subspace.} as well as dealing with errors in data, e.g., noise~\cite{candes2010matrix}, missed entries~\cite{candes2009exact}, corruptions~\cite{Candes:RPCA}, and outliers~\cite{xu2010robust}. Among them, sparse subspace clustering (SSC)~\cite{SSC:PAMI2014} and low-rank representation (LRR)~\cite{LRR:PAMI2014} stand out as two most popular methods, which formulate the discovery of multi-subspace structure as finding a sparse or low-rank representation of the data samples using data themselves as the dictionary (basis). It has been shown in literatures that these methods can achieve more accurate subspace structure than single subspace analysis methods. However, both SSC and LRR have the following major drawbacks: First, they use the original data as the basis rather than learning the basis explicitly. It is problematic when data contain errors yet are still used for reconstruction or/and clustering. Second, neither sparse representation via $l_1$ regularization nor low-rank representation in a global constraint provides a direct description for each data sample. In other words, data samples cannot find the subspace where they lie, and their contained errors are unable to be removed. Thirdly, SSC and LRR suffer from the computational complexity that is quadratic and cubic to the sample size, respectively. 

\textbf{Our work:} In this paper, to tackle all above problems, we aim to simultaneously \emph{learn the basis} for each subspace and \emph{directly generate sparse representation} for each data sample, as well as \emph{removing errors} (e.g., random corruptions\footnote{A fraction of random entries of data are grossly corrupted.} and sample-specific outliers\footnote{A fraction of data are far away from their respective subspaces.}) in the data \emph{in an efficient way}. To this end, we propose a novel column $L_0$-norm constrained matrix factorization (MFC$_0$) method, and develop a first-order alternating direction algorithm to stably and efficiently solve the nonsmooth and nonconvex objective function of MFC$_0$. 

Specifically, given a collection of data approximately generated from a union of independent subspaces with an equal subspace dimension, we explore their subspace structure from the \emph{matrix factorization} perspective in the following way: 
First, we restrict the basis matrix to be orthonormal, which guarantees that the learnt basis indeed span multiple subspaces.
Then, we pursue a direct sparse representation of each data sample on only the basis that form its underlying subspace. 
To achieve this goal, we leverage an \emph{$l_0$-norm} to confine the number of nonzeros elements of each column of the representation matrix. 
Furthermore, to handle different types of errors contained in the data, we incorporate different regularizations. 

With above constraints, the objective function of MFC$_0$ is nonconvex on the coupled basis and representation matrix, and is nonsmooth on the $l_0$-norm constraint. Thus, it is a very challenging optimization problem that cannot be solved via typical gradient (or subgradient) based methods. To overcome the nonconvex problem, we develop a first-order optimization algorithm motivated by ADMM~\cite{boyd2011distributed} to split multivariable objective function into several univariate subproblems and solve each subproblem one by one. To tackle the subproblem on nonsmooth $l_0$-norm constraint, we design a novel proximal operator, which is inspired by recent proximal algorithms~\cite{parikh2013proximal}, and solve it efficiently and analyticly. 

After learning basis and representation matrix, MFC$_0$ can accomplish many tasks, e.g., subspace clustering, data reconstruction, error correction, etc.. We evaluate MFC$_0$ on both synthetic data and real-world datasets. Experimental results demonstrate that besides the superiority and efficiency over traditional and state-of-the-art methods for subspace clustering, data reconstruction, and error correction, MFC0 also demonstrates its uniqueness for multi-subspace basis learning and direct representation learning. 


In summary, our key contributions are as follows:
\begin{itemize}
\item We propose a novel MFC$_0$ method from the matrix factorization perspective to perform multi-subspace analysis. To the best of our knowledge, this is the first work that considers simultaneously learning the basis, generating direct sparse representation, and correcting errors of data generated from multiple subspaces. 
\item We develop a first-order alterating direction algorithm to stably and effectively solve the nonconvex objective function of MFC$_0$. 
\item We design a novel proximal operator to analyticly and efficiently solve the nonsmooth $l_0$-norm constraint.
\item Experimental results demonstrate that MFC$_0$ shows high resistance to errors,  
achieves high subspace clustering performance, learns multi-subspace basis, and generates direct sparse representation.
\end{itemize}

\section{Related Work}

\myparatight{Matrix factorization and its variants}
Matrix Factorization (MF) has been studied for several decades. Classical methods, such as Principal Component Analysis (PCA), Linear Discriminant Analysis (LDA), Independent Component Analysis (ICA), are exemplars of low-rank matrix factorizations. Afterwards, to solve specific problems, many MF variants, including Concept MF~\cite{liu2014constrained}, Maximum Margin MF~\cite{srebro2004maximum}, Convex MF~\cite{ding2010convex}, and Online MF~\cite{mairal2010online}, etc., are proposed. These methods share in common that there is no constraint on the sign of elements of factorized matrices. In contrast, a paradigm of factorization, called NMF, was first proposed in~\cite{lee1999learning} and has been widely used for its \emph{sparse, parts-based, and additive representation}. NMF is suitable for many research fields, e.g., data mining, image processing, pattern recognition, and machine learning. Its real-world applications include document clustering~\cite{xu2003document}, image segmentation~\cite{sandler2011nonnegative}, image recognition~\cite{zafeiriou2006exploiting,wang2013graph}, speech and audio processing~\cite{fevotte2009nonnegative}, blind source separation~\cite{cichocki2006new}, to name a few. One can refer to the comprehensive review of NMF in~\cite{wang2013nonnegative}.

\myparatight{Approximate sparsity via $l_1$-norm} 
Sparsity in classic NMF occurs is as a by-product rather than a designed objective. To this end, researchers have paid much attention to impose sparsity on NMF explicitly. These methods either penalize~\cite{eggert2004sparse,kim2007sparse} or constrain~\cite{hoyer2004non} $l_1$-norm of the representation matrix to yield a sparse representation~\cite{donoho2003optimally}. As a convex relaxation of $l_0$-norm, the exact sparsity measure which renders NP-hard combinational optimization problem, $l_1$-norm is used to improve computational feasibility and efficiency of sparse representation~\cite{bruckstein2009sparse}. However, as NMF is NP-hard itself~\cite{vavasis2009complexity}, we can say that any algorithm for NMF is suboptimal. Thus, a heuristic $l_0$-norm constrained NMF might be as appropriate and efficient as $l_1$-norm sparse NMF. 

\myparatight{$L_0$-norm based sparsity} 
Little work has been done using direct $l_0$-norm, the exact sparsity measure which causes the NP-hard combinational optimization problem. K-SVD~\cite{aharon2006img} aims to find an overcomplete dictionary for sparse representation, and the sparse approximation is based on an $l_0$-norm inequality constraint and is solved via greedy orthogonal matching pursuit (OMP)~\cite{tropp2007signal}. 
Peharz and Pernkopf~\cite{peharz2012sparse} further proposed sparse NMF that introduced an $l_0$-norm to constrain the basis or representation matrix of approximate NMF. It is solved via sparse alternating nonnegative least-squares (sANLS).
Wang et al.~\cite{wang2014hierarchical} leveraged hierarchical Bayes to model an adaptive sparseness prior, which has a similar properity with an $l_0$-norm.
However, these methods are under the hidden assumption that data are generated from single subspace. In addition, they cannot handle various errors. In this work, we focus on analyzing data (approximately) drawn from multiple subspaces. Different from using greedy methods, the $l_0$-norm constraint in MFC$_0$ is solved analytically via our designed proximal operator, which is easy to be incorporated into our developed alternating direction algorithm. 

\section{Notations and Problem Definition}

\subsection{Main Notations}
In this paper, vectors and matrices are written as bold lowercase and uppercase symbols. The $i$-th entry of vector $\mathbf{u}$ is $u_i$. For matrix $\mathbf{W}$, its $(i,j)$-th entry, $i$-th column, and $j$-th row are denoted as $w_{i,j}$, $\mathbf{w}_i$, and $\mathbf{w}^j$ respectively. 
The horizontal and vertical concatenation of a set of $K$ matrices $\mathbf{W}_1, \mathbf{W}_2, \cdots, \mathbf{W}_K$ along row and column are denoted by $\left[\mathbf{W}_1, \mathbf{W}_2, \cdots, \mathbf{W}_K \right]$ and $\left[\mathbf{W}_1; \mathbf{W}_2; \cdots; \mathbf{W}_K \right]$, respectively.
The superscript $T$ stands for the transpose. 

For vector $\mathbf{u} \in \Re^m$, its $l_p$-norm is $\left\| \mathbf{u} \right\|_p = {(\sum_{i=1}^m |u_i|^p)}^\frac{1}{p}$, and pseudo $l_0$-norm is the number of nonzero entries, i.e., $\left\| \mathbf{u} \right\|_0 = \#\{i:u_i \neq 0\}$. For matrix $\mathbf{W} \in \Re^{m \times n}$, its Frobenius norm is $\left\| \mathbf{W} \right\|_F =
\sqrt{\sum_{j=1}^m \left\| \mathbf{w}^j\right\|_2^2}$; $l_{2,1}$-norm is $\left\| \mathbf{W} \right\|_{2,1}=\sum_{i=1}^n \left\| \mathbf{w}_i \right\|_2$; $l_1$-norm is $\left\| \mathbf{W} \right\|_1=\sum_{j=1}^m \sum_{i=1}^n |w_{ij}| $; and $l_\infty$-norm is $\left\| \mathbf{W} \right\|_\infty = \max_{i,j}|w_{i,j}|$. The inner product between two matrices $\mathbf{W}$, $\mathbf{U}$ is $\langle \mathbf{W}, \mathbf{U} \rangle = \textbf{tr}(\mathbf{W}^T \mathbf{U})$, where $\textbf{tr}$ is the trace operator defined on a square matrix.

We denote $k$-th subspace as $\mathcal{S}_k$ and a collection of $K$ subspaces as $\{\mathcal{S}_k\}_{k=1}^K$.

\subsection{Problem Definition}

We present two definitions related to our problem at first. 
\begin{itemize}[noitemsep,nolistsep]
\item \myparatight{Union and sum of subspaces} The union of $K$ subspaces $\{\mathcal{S}_k\}_{k=1}^K$ is defined as $\cup_{k=1}^K \mathcal{S}_k = \{ \mathbf{t}: \mathbf{t} \in \mathcal{S}_k, \forall k \}$. The sum of $K$ subspaces $\{\mathcal{S}_k\}_{k=1}^K$ is defined as $\sum_{k=1}^K \mathcal{S}_k = \{ \mathbf{t}: \mathbf{t} = \sum_{k=1}^K \mathbf{t}_k, \mathbf{t}_k \in \mathcal{S}_k \}$.    

\item \myparatight{Independent subspaces} $K$ subspaces $\{\mathcal{S}_k\}_{k=1}^K$ are independent if and only if $\mathbf{S}_k \cap \sum_{j \neq k} \mathcal{S}_j = \{ 0 \}$. Or to say, $\text{dim}(\sum_{k=1}^K \mathcal{S}_k) = \sum_{k=1}^K \text{dim} (\mathcal{S}_k) $. 
\end{itemize}

With above definitions, our problem is defined as
\begin{problem}[Robust Multi-subspace Analysis]
Given a collection of data samples strictly/approximately drawn from a union of independent subspaces with an equal subspace dimension, robust multi-subspace analysis is to learn the basis for each subspace, learn the associated representation for each sample upon the basis, or/and correct the error for each data sample.
\end{problem}

Formally, given $n$ data samples $ \mathbf{Z} = \left\{\mathbf{z}_i\right\}_{i=1}^n \in \Re^{m \times n}$ that are generated from a basis matrix $\tilde{\mathbf{X}} = \left\{\tilde{\mathbf{x}}_i\right\}_{i=1}^d \in \Re^{m \times d}$ that span $K$ independent subspaces $\{\mathcal{S}_k\}_{k=1}^K$, with a corresponding representation matrix $\tilde{\mathbf{Y}} = \left\{\tilde{\mathbf{y}}_i\right\}_{i=1}^n \in \Re^{d \times n}$, or/and contaminated with errors $\tilde{\mathbf{E}} \in \Re^{m \times n} $ 
\begin{equation}
\label{a01}
\mathbf{Z} = \tilde{\mathbf{X}} \tilde{\mathbf{Y}} + \tilde{\mathbf{E}}.
\end{equation}

Then, our purpose is from the observed data $\mathbf{Z}$ to learn the basis $\tilde{\mathbf{X}}_k$ for each subspace $\mathcal{S}_k$, the representation $\tilde{\mathbf{Y}}_k$ upon the basis $\tilde{\mathbf{X}}_k$, or/and the error $\tilde{\mathbf{E}}_k$ such that
\begin{small}
\begin{equation}
\label{a10}
\left[ \mathbf{Z}_1, \cdots, \mathbf{Z}_K \right] = \left[ \tilde{\mathbf{X}}_1, \cdots, \tilde{\mathbf{X}}_K \right] 
\left[
     \begin{array}{ccc}
    \tilde{\mathbf{Y}}_1 \\
    \vdots   \\
    \tilde{\mathbf{Y}}_K \\
  \end{array}
\right] 
+ \left[ \tilde{\mathbf{E}}_1, \cdots, \tilde{\mathbf{E}}_K \right],
\end{equation}
\end{small}
where $\small \mathbf{Z}= \left[ \mathbf{Z}_1, \cdots, \mathbf{Z}_K \right]$ with $\small \mathbf{Z}_k \in \Re^{m \times n_k}$ merely collecting the samples from the $k$-th subspace $\mathcal{S}_k$ and $\sum_{k=1}^K n_k = n$; $\small \tilde{\mathbf{E}} = \left[ \tilde{\mathbf{E}}_1, \cdots, \tilde{\mathbf{E}}_K \right]$ with $\small \tilde{\mathbf{E}}_k \in \Re^{m \times n_k}$ containing the errors of samples in $\small \mathbf{Z}_k$; $\small \tilde{\mathbf{X}} = \left[ \tilde{\mathbf{X}}_1, \cdots, \tilde{\mathbf{X}}_K \right] \in \Re^{m \times K d_0}$ with $\small \tilde{\mathbf{X}}_k \in \Re^{m \times d_0}$ the basis to span $\mathcal{S}_k$; $\small \tilde{\mathbf{Y}} = \left[ \tilde{\mathbf{Y}}_1; \cdots; \tilde{\mathbf{Y}}_K \right]$ with $\small \tilde{\mathbf{Y}}_k \in \Re^{d_0 \times n}$ the representation upon the basis $\small \tilde{\mathbf{X}}_k$; $d_0$ indicates the subspace dimension, $d=K d_0$.

In the sequel, we first propose a model to deal with data that are clean, i.e., $\tilde{\mathbf{E}}=\mathbf{0}$. Then we generalize our model to handle data that are contaminated with random corruptions or sample-specific outliers, i.e., $\tilde{\mathbf{E}} \neq \mathbf{0}$.




\section{The Proposed Method}

In this section, we first introduce a theorem that inspires us to formulate the problem with clean data. Then, we generalize the model to handle contaminated data. Finally, we present the algorithm to solve our model.  



\subsection{Problem Formulation}


To start with, we assume that $n$ data samples $\mathbf{Z}$ are strictly generated from $K$ independent subspaces. 
In doing so, we can come to the conclusion as follows:
\begin{theorem}
\label{theorem_1}
Suppose $K$ subspaces $\left\{\mathcal{S}_k\right\}_{k=1}^K$ are independent and $\tilde{\mathbf{X}}_k$ consists of basis that only span subspace $\mathcal{S}_k$, then the solution of representation matrix $\tilde{\mathbf{Y}}^{\star}$ to Eq.($\ref{a10}$) dealing with clean data is block-diagonal:
\begin{small}
\begin{equation}
\label{a11}
\tilde{\mathbf{Y}}^{\star} =
\left[
     \begin{array}{ccc}
    \mathbf{\hat{Y}}_1 & \cdots & \mathbf{0} \\
    \vdots   & \ddots & \vdots \\
    \mathbf{0} &  \cdots & \mathbf{\hat{Y}}_K \\
  \end{array}
\right], 
\end{equation}
\end{small}
where $\mathbf{\hat{Y}}_k \in \Re^{d_0 \times n_k}$ and $\tilde{\mathbf{Y}}^k = [\mathbf{0}, \cdots, \mathbf{\hat{Y}}_k, \mathbf{0}, \cdots]$.
\end{theorem}

\begin{proof}
See Appendix~\ref{app:theorem1}.
\end{proof}

Theorem~\ref{theorem_1} shows the condition under which we can achieve the \emph{block-diagonal} representation matrix, which is of vital importance to support analyzing the structure of multi-subspace. That is, if we know the basis that clean data lie in, then each block of the representation matrix  characterizes each subspace. 
In practice, however, the basis of each subspace are often unknown, Thus, the generated representation matrix cannot be guaranteed to be block-diagonal. Fortunately, Theorem~\ref{theorem_1} offers us the hint to construct a satisfactory model. 
In short, we expect to \emph{simultaneously learn the basis matrix and learn the representation matrix inspired by Theorem~\ref{theorem_1}}. 

To be more specific, we first impose an orthonormal constraint on the basis matrix to ensure that its columns can indeed form the ``basis" to span independent subspaces. 
In order to generate the block-diagonal representation, we characterize each data sample as a nonnegative combination\footnote{We only deal with data with nonnegative entries.} of only the basis that span its underlying subspace. To this end, we leverage a \emph{column $l_0$-norm} constrained on the representation matrix and set its number of nonzero elements equal to the subspace dimension. 

Formally, the objective function satisfying above constraints is defined as
\begin{small}
\begin{align}
& \min \limits_{\mathbf{X} ,\mathbf{Y}} {\left \| \mathbf{Z-XY} \right \|}_F^2 \label{a12-1} \\ 
& \textrm{s.t.} \quad \mathbf{X^TX}=\mathbf{I}_d,  \label{a12-2} \\ 
& \qquad \mathbf{Y} \geq 0, \label{a12-3} \\ 
& \qquad \| \mathbf{y}_i \|_0 = d_0, \forall i, \label{a12-4} 
\end{align}
\end{small}
where Eq.(\ref{a12-2}) is the orthonormal constraint on $\mathbf{X}$; $\mathbf{I}_d$ is a $d \times d$ identity matrix. Eq.(\ref{a12-3}) indicates a nonnegative combination of the orthonormal basis; ``$\geq$" is taken component-wise. Eq.(\ref{a12-4}) is the requirement of the number of nonzero coefficients. For writing simplicity, we omit the ``$ \forall i $'' symbol in the subsequent related equations.


In practice, the observed data are often contaminated with errors, e.g., noises, random corruptions or sample-specific outliers, rendering the data away from their exact underlying subspaces. To handle different types of errors, we adopt different regularizations. Specifically, we generalize the objective function in Eq.(\ref{a12-1}) to
\begin{small}
\begin{equation}
\label{a13}
\begin{split}
& \min \limits_{\mathbf{X,Y,E}} {\left\| \mathbf{Z-XY-E} \right\|}_F^2 + \lambda \left\| \mathbf{E} \right\|_{\Delta} \\
& \textrm{s.t.} \quad \mathbf{X^TX}=\mathbf{I}_d, \, \mathbf{Y} \geq 0, \, \| \mathbf{y}_i \|_0 = d_0,
\end{split}
\end{equation}
\end{small}
where $\left\| \mathbf{E} \right\|_{\Delta}$ is some type of matrix norm. For random corruptions, it is $\left\| \mathbf{E} \right\|_1$; While for sample-specific outliers, it is $\left\| \mathbf{E} \right\|_{2,1}$. $\lambda>0$ is the regularization parameter that controls the effect of two terms.

By implementing above two objective functions, we claim that when the data are clean, then $\mathbf{Y}$ is block-diagonal after reordering its columns and rows, i.e., $\mathbf{Y} = \tilde{\mathbf{Y}}^{\star}$. When the data are contaminated with random corruptions or sample-specific outliers to some extent, our experimental results show that $\mathbf{Y}$ is still approximate block-diagonal.

It is necessary to notice that SSC and LRR can also be boiled down to the matrix factorization framework. They share some similarities with the proposed method, but there exist two major differences: On one hand, they take advantage of what they refer to as ``self-expressiveness" property and use the data themselves to form the dictionary (basis). However, leveraging the data as the basis fail to perform data reconstruction when they contain errors to some degree (See Figure $\ref{fig1}$). On the other hand, they respectively utilize an indirect $l_1$-norm and nuclear norm to obtain the sparse and low-rank representation of the data, which however, are unable to know the exact basis that are used to represent each data sample. Thus, their ability for subspace clustering is also limited (See Figure $\ref{fig5}$).

\begin{figure}[!hbt]
\begin{center}
\subfigure[Random Corruptions]{
\includegraphics[width=0.24\linewidth]{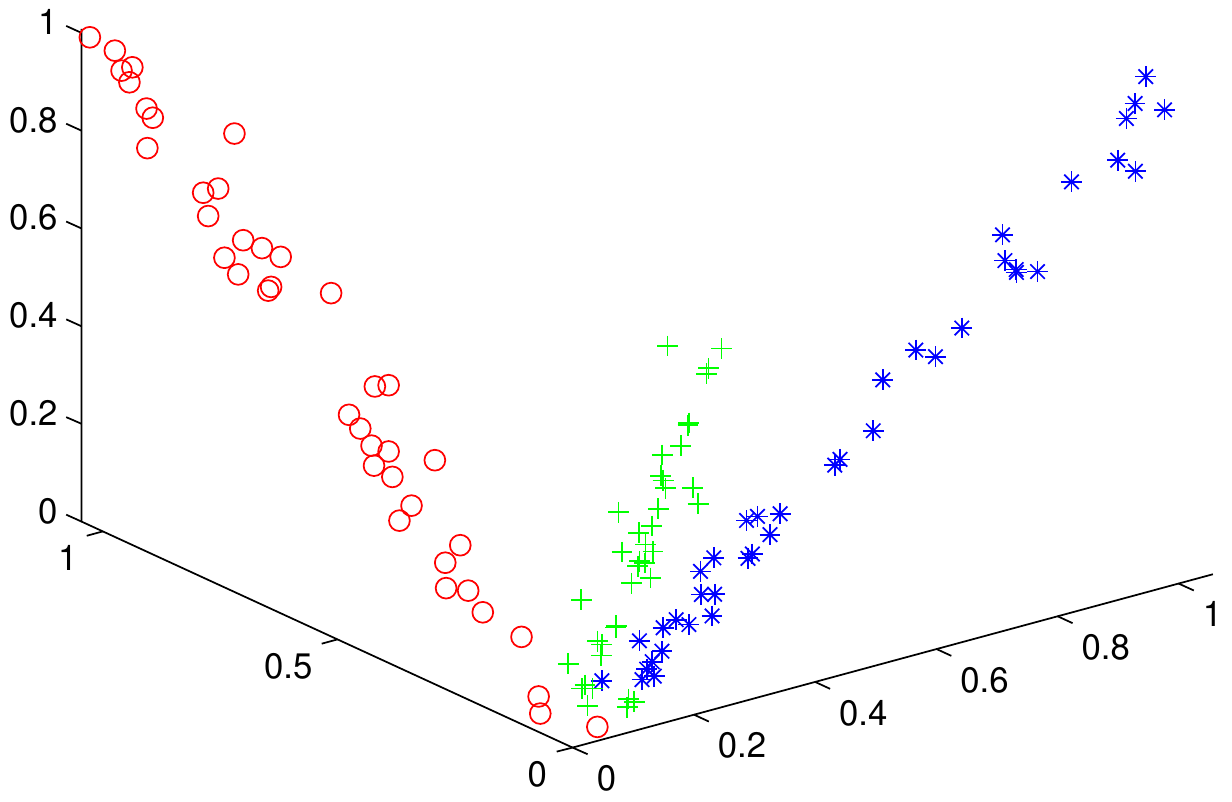}
\includegraphics[width=0.24\linewidth]{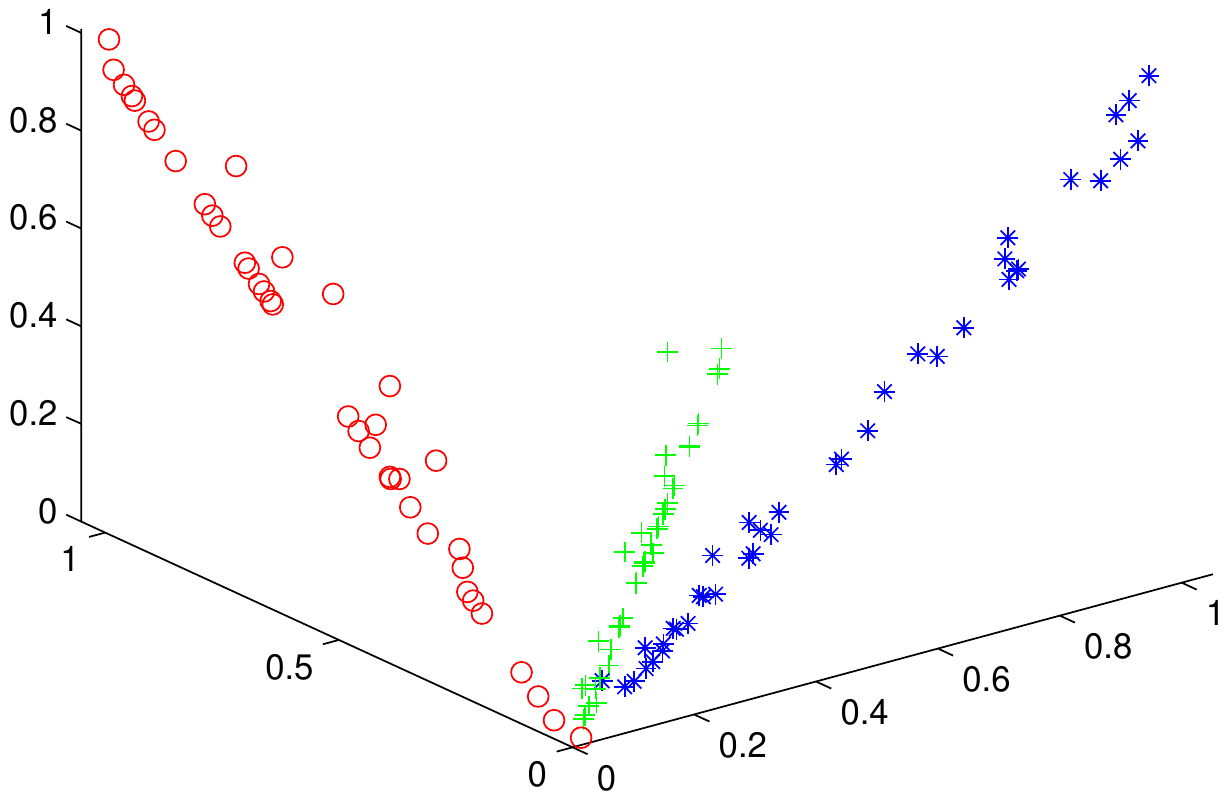}
\includegraphics[width=0.24\linewidth]{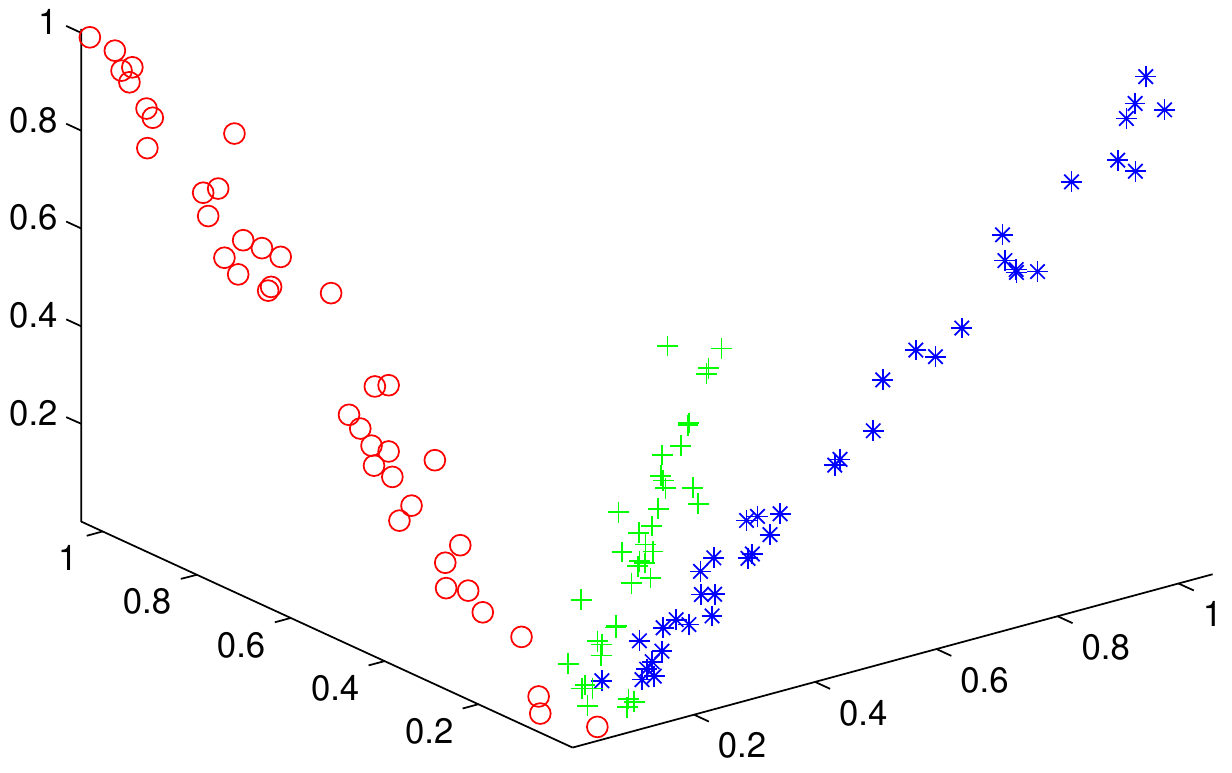}
\includegraphics[width=0.24\linewidth]{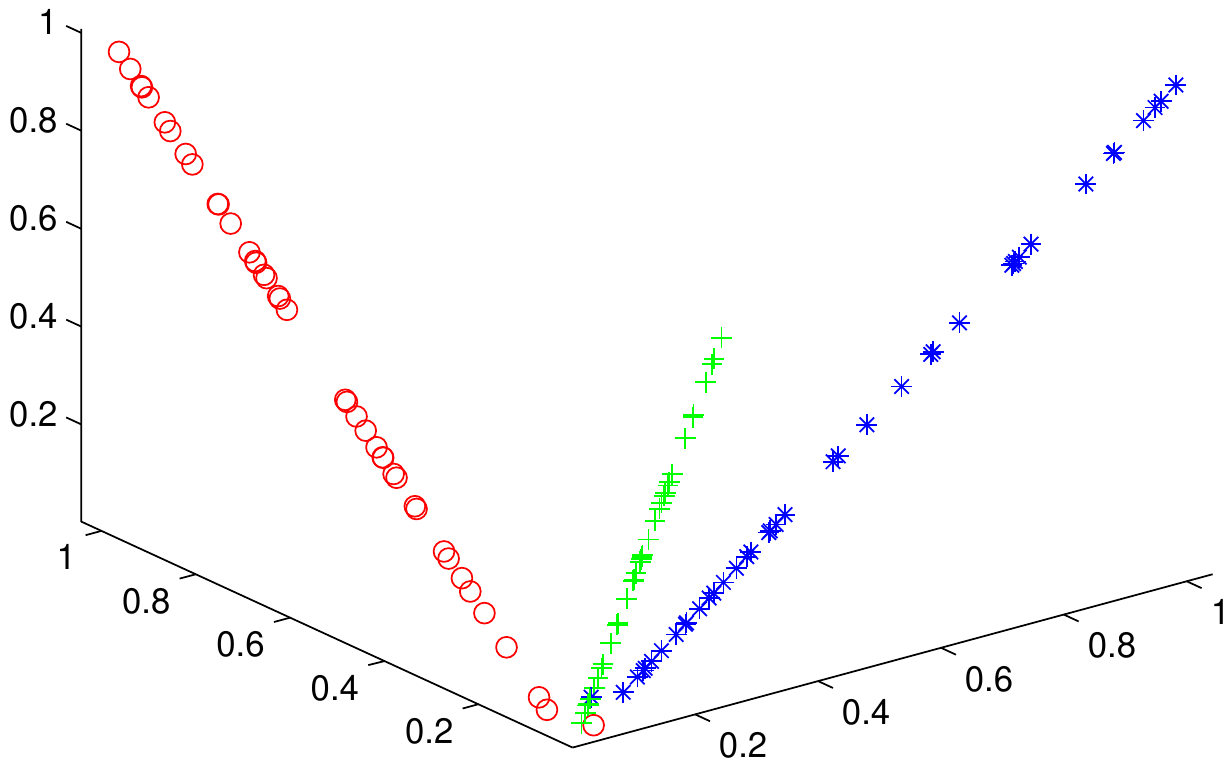}
}
\subfigure[Sample-Specific Outliers]{
\includegraphics[width=0.24\linewidth]{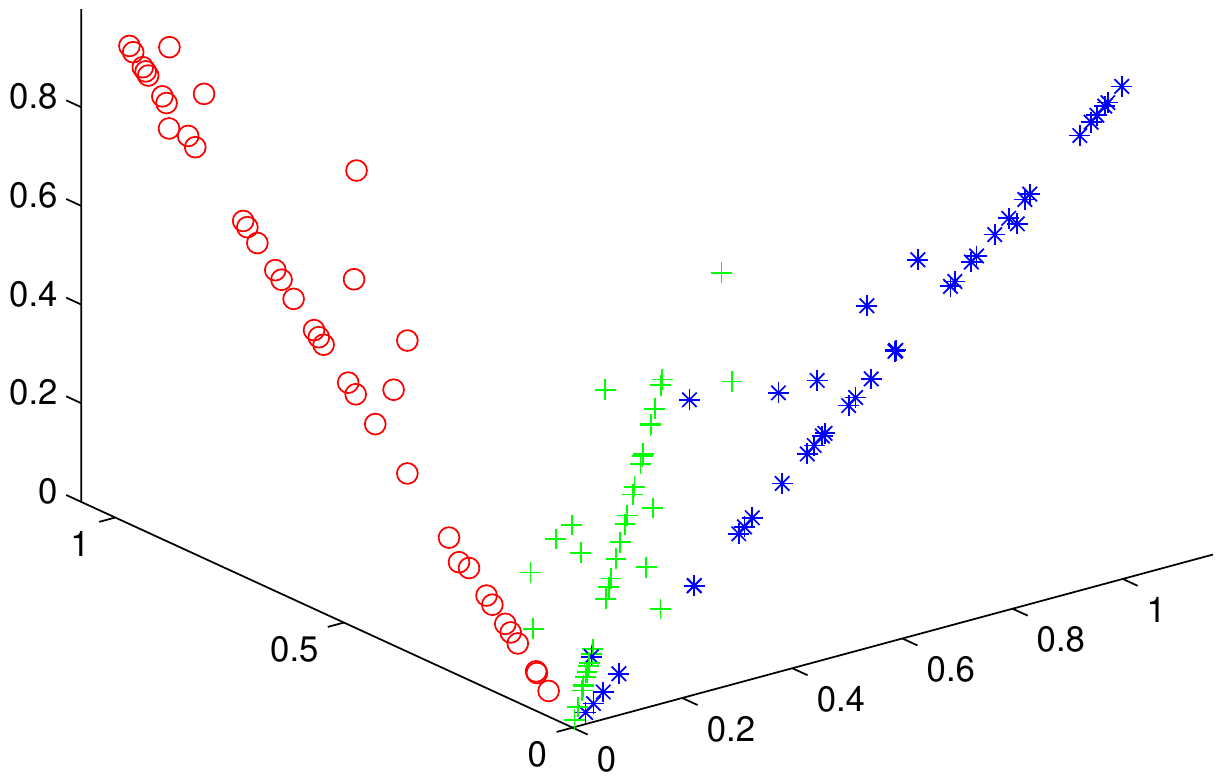}
\includegraphics[width=0.24\linewidth]{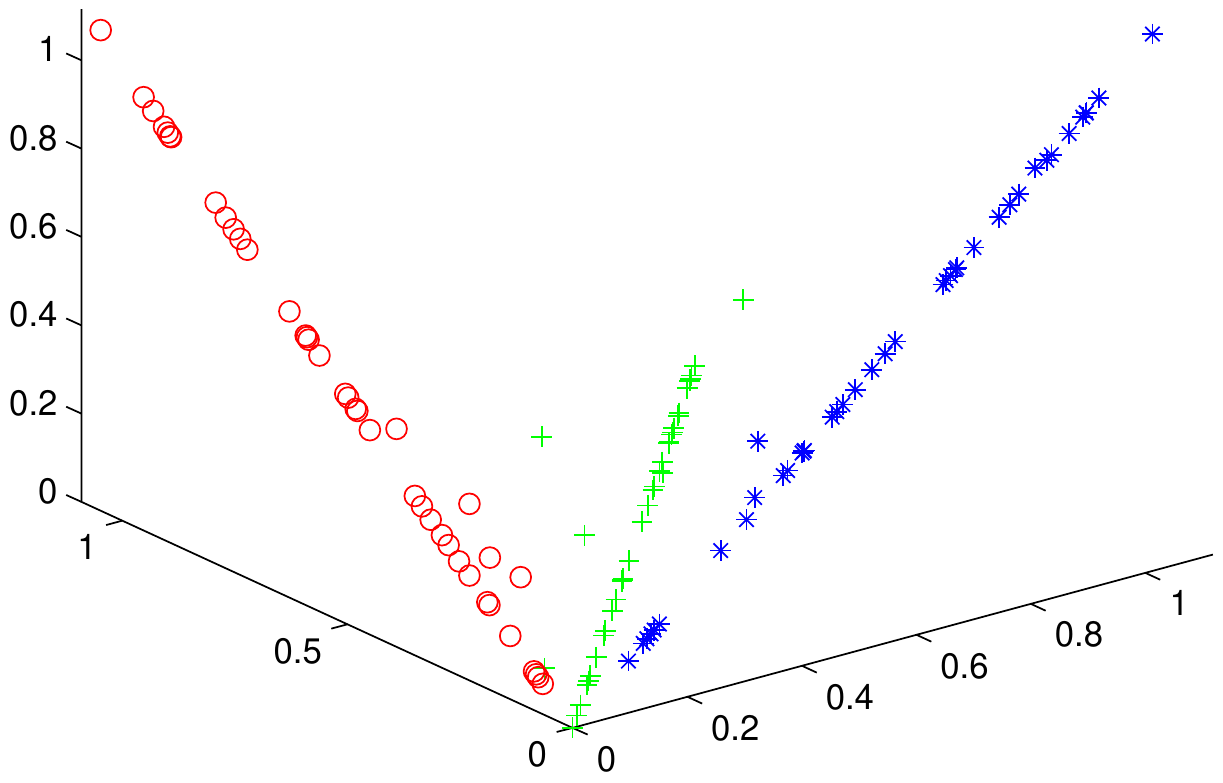}
\includegraphics[width=0.24\linewidth]{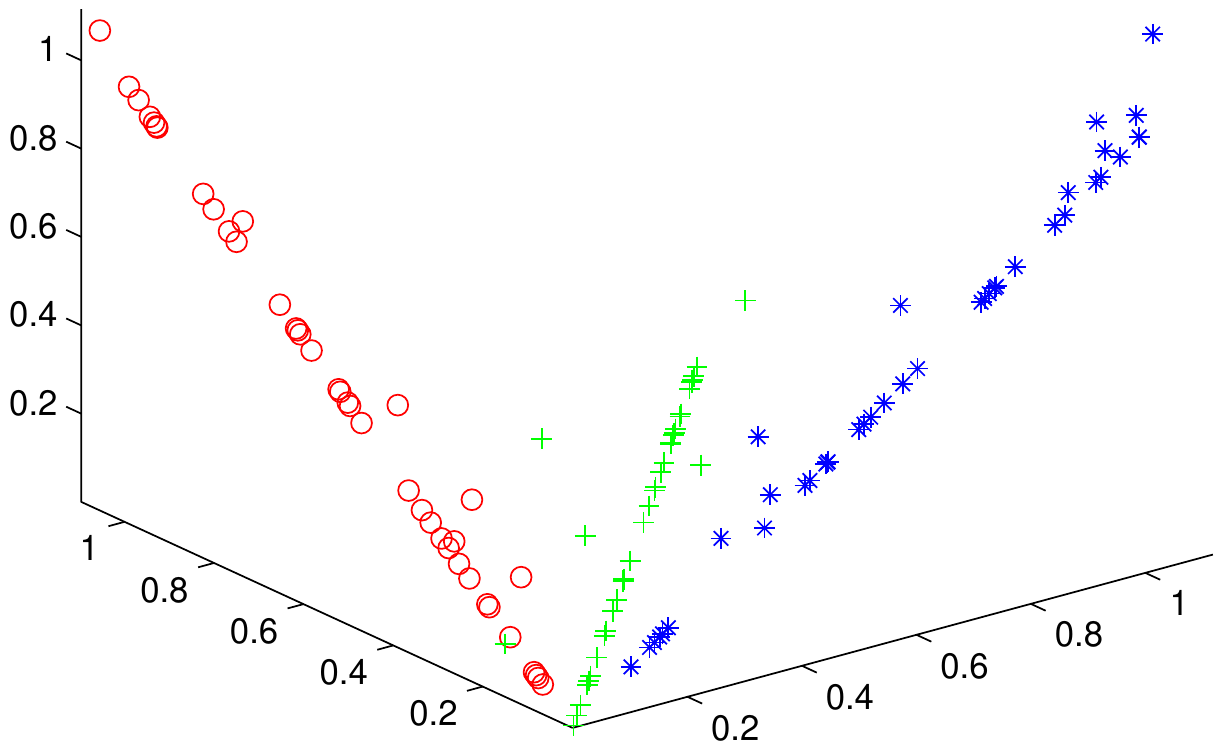}
\includegraphics[width=0.24\linewidth]{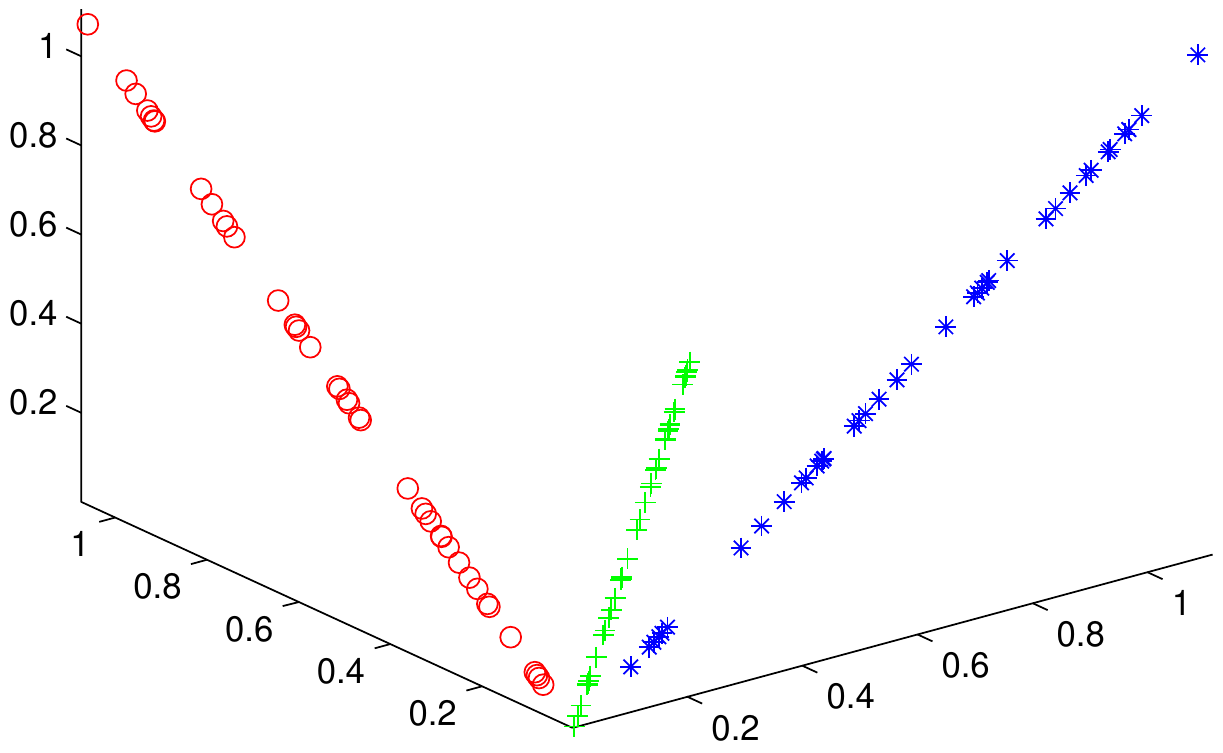}
}
\end{center}
\caption{Reconstruction results of data generated from 3 independent subspaces and contaminated with (a) random corruptions
and (b) sample-specific outliers. The first column shows the original contaminated data. The remaining three columns are the results of SSC, LRR, and MFC$_0$ respectively.}
\label{fig1}
\end{figure}

\subsection{Problem Optimization}

In this part, without loss of generality, we just solve for the general case in Eq.(\ref{a13}). Notice that we need to handle the \emph{nonconvex} multivariable objective function, as well as the \emph{nonsmooth} $l_0$-norm constraint. Fortunately, for function with multivariables, the ADMM method is often the choice~\cite{boyd2011distributed}. Moreover, proximal algorithms has recently been a popular tool for solving nonsmooth and nonconvex problems. It's basic operation is evaluating the proximal operator of a univariable function, which often admits a closed-from solution or can be solved very quickly~\cite{parikh2013proximal}.  Inspired by such superiority, we develop a first-order alternating direction algorithm to solve Eq.(\ref{a13}). The overall procedure of the algorithm is first introducing auxiliary variables and quadratic penalties into the objective function, then iteratively minimizing the augmented Lagrangian function with respect to each primal variable, and finally updating the multipliers. To tackle the subproblem on nonsmooth $l_0$-norm constraint, we design a novel proximal operator to solve it efficiently and analytically.

We begin with introducing an auxiliary variable $\mathbf{V}$ to variable $\mathbf{Y}$ which exists in both the object function and the constraint. We reformulate Eq.(\ref{a13}) as
\begin{equation}
\label{a21}
\begin{split}
& \min \limits_{\mathbf{X,Y,E,V}} \left\| \mathbf{Z-XY-E} \right\|_F^2 + \lambda \left\| \mathbf{E} \right\|_{\Delta}, \\
& \qquad \textrm{s.t.} \quad \mathbf{X^TX}=\mathbf{I}_d, \mathbf{Y} = \mathbf{V}, \mathbf{V} \geq 0, \| \mathbf{v}_i \|_0 = d_0.
\end{split}
\end{equation}

Then, the augmented Lagrangian function of Eq.(\ref{a21}) is
\begin{small}
\begin{equation}
\label{a22}
\begin{split}
& \mathcal{L_A}(\mathbf{X,Y,V,E,P}) = \left\| \mathbf{Z-XY-E} \right\|_F^2 + \lambda \left\| \mathbf{E} \right\|_{\Delta} \\
& \qquad \quad + \langle \mathbf{P}, \mathbf{Y-V} \rangle + \frac{\beta}{2} \left \| \mathbf{Y}-\mathbf{V} \right \|_F^2, \\
& \quad \textrm{s.t.} \quad  \mathbf{X^TX}=\mathbf{I}_d, \mathbf{V} \geq 0, \, \| \mathbf{v}_i \|_0 = d_0,
\end{split}
\end{equation}
\end{small}
where $\mathbf{P}$ is the Lagrangian multiplier, $\beta >0$ is the quadratic penalty parameter. Note that adding the penalty term does not change the optimal solution, since any feasible solution satisfying the constraint in Eq.($\ref{a22}$) vanishes the penalty term.

Finally, our alternating direction algorithm consists of iteratively minimizing Eq.(\ref{a22}) with respect to one of $\mathbf{X,Y,V,E}$ while fixing the others, and updating the multiplier $\mathbf{P}$. Specifically, we solve for variables and the multiplier step by step. 

$\textbf{Update X}$: By discarding terms that are irrelevant to $\mathbf{X}$, we rewrite the subproblem with respect to $\mathbf{X}$ as
\begin{equation}
    \label{a23}
    \min_\mathbf{X} \left\| \mathbf{Z}-\mathbf{XY} -\mathbf{E} \right\|_F^2, \quad \textrm{s.t.} \quad \mathbf{X^TX}=\mathbf{I}_d.
\end{equation}
We denote the singular value decomposition (SVD) of $\mathbf{(Z-E)Y^T}=\mathbf{L \Sigma R^T}$, then from $\textbf{Theorem 2}$ below, we can achieve the closed-form solution of $\mathbf{X}$. That is,
\begin{equation}
    \mathbf{X} = \mathbf{LR^T}.
\end{equation}

\begin{theorem}
\label{theorem_2}
Let $\mathbf{A} \in \Re^{m \times n}$ and $\mathbf{B} \in \Re^{d \times n}$ be any two matrices. Denote the SVD of $\mathbf{AB^T}$ as $\mathbf{AB^T}= \mathbf{L \Sigma R^T}$. Then the orthonormal constrained minimization problem
\begin{equation*}
    \min_{\mathbf{D}} \left\| \mathbf{A - DB}\right\|_F^2   \quad \textrm{s.t.}  \quad \mathbf{D^TD}=\mathbf{I}_d, 
\end{equation*}
has an analytic solution $\mathbf{D}=\mathbf{LR^T}$.
\end{theorem}
\begin{proof}
See Appendix~\ref{app:theorem2}.
\end{proof}

$\textbf{Update Y}$: The subproblem of $\mathbf{Y}$ is
\begin{equation}
\label{g}
\min \limits_\mathbf{Y} \left\| \mathbf{Z-XY-E} \right\|_F^2 + \langle \mathbf{P, Y-V} \rangle + \frac{\beta}{2} \left\| \mathbf{Y-V} \right\|_F^2.
\end{equation}

Setting the derivative with respect to $\mathbf{Y}$ to zero, rearranging the terms, and using the constraint $\mathbf{X^TX}=\mathbf{I}_d$ yields

\begin{equation}
\label{h}
\mathbf{Y} = (1+\beta)^{-1}\big(\mathbf{X}^T \mathbf{(Z-E)} + \beta \mathbf{V} - \mathbf{P}\big).
\end{equation}

$\textbf{Update E}$: The subproblem of $\mathbf{E}$ becomes
\begin{equation}
\label{i}
\min \limits_\mathbf{E} \left\| \mathbf{E-(Z-XY)} \right\|_F^2 + \lambda \left\| \mathbf{E} \right\|_{\Delta}.
\end{equation}
For sample-specific outliers, we set $\left\| \mathbf{E} \right\|_{\Delta} = \left\| \mathbf{E} \right\|_{2,1}$. Then the solution to Eq.(\ref{i}) is equivalent to solving the proximal operator for $l_{2,1}$-norm. From~\cite{yang2009fast}, $\mathbf{E}$ can be obtained with a closed-form. Concretely, denote $\mathbf{G}=\mathbf{Z-XY}$, then
\begin{equation}
\label{j}
    \mathbf{e}_i =
        \begin{cases}
            (1-\frac{\lambda/2}{\left\| \mathbf{g}_i \right\|_2}) \mathbf{g}_i  & \text{$\left\| \mathbf{g}_i \right\|_2 \geq \frac{\lambda}{2}$}, \\
            0 & \text{$\left\| \mathbf{g}_i \right\|_2 < \frac{\lambda}{2}$}. \\
        \end{cases}
\end{equation}
Similarly, for random corruptions, we set $\left\| \mathbf{E} \right\|_{\Delta} = \left\| \mathbf{E} \right\|_{1}$. The solution then equals to finding the proximal operator for $l_1$-norm, and can be achieved via the Soft-Thresholding Operator~\cite{murphy2012machine}
\begin{equation}
    \mathbf{E} =
    \begin{cases}
           \mathbf{G} - \frac{\lambda}{2}  & \text{$\mathbf{G} \geq \frac{\lambda}{2}$}, \\
           \mathbf{0} & \text{$-\frac{\lambda}{2} \geq \mathbf{G} \leq \frac{\lambda}{2}$}, \\
           \mathbf{G} + \frac{\lambda}{2}  & \text{$\mathbf{G} \leq -\frac{\lambda}{2}$}. \\
    \end{cases}
\end{equation}

$\textbf{Update V}$: The subproblem associated with $\mathbf{V}$ is
\begin{equation}
\label{k}
\begin{split}
& \min_{\mathbf{V}} \left\| \mathbf{V}-(\mathbf{Y}+\beta^{-1} \mathbf{P}) \right\|_F^2, \\
& \quad s.t \quad  \mathbf{v}_i \geq 0, \left\| \mathbf{v}_i \right\|_0 = d_0.
\end{split}
\end{equation}

Here, we emphasize that due to the discrete $l_0$-norm constraint in Eq.(\ref{k}), $\mathbf{V}$ cannot be obtained via gradient or subgradient based methods. Benefiting from aforementioned proximal algorithms~\cite{parikh2013proximal}, we can design a simple yet very efficient algorithm to solve $\mathbf{V}$ column by column. The detail is as follows:

We denote $\mathbf{U} = (\mathbf{Y}+ \beta^{-1} \mathbf{P})$ and define the operator $\mathrm{P}_{d_0}:\Re^d \rightarrow \Re^d$ for $\mathbf{u}_i$ as
\begin{equation*}
    \mathrm{P}_{d_0}(\mathbf{u}_i) = \argmin_{\mathbf{v}_i} \left\{ || \mathbf{u}_i - \mathbf{v}_i ||_2^2 : ||\mathbf{v}_i||_0 = d_0 \right\}.
\end{equation*}

We also define the nonnegative orthant mapping as
\begin{equation*}
    \mathrm{P}_{+}(\mathbf{u}_i) = \argmin_{\mathbf{v}_i} \left\{ || \mathbf{u}_i - \mathbf{v}_i ||_2^2 : \mathbf{v}_i \geq 0 \right\} = \max\left\{0, \mathbf{u}_i\right\}.
\end{equation*}
and introduce an indicator function $I_{\mathcal{V}}(\mathbf{v})$ over the set 
\begin{equation*}
\mathcal{V} = \left\{\mathbf{v} \in \Re^d: \mathbf{v} \geq 0, ||\mathbf{v}||_0 = d_0\right\}. 
\end{equation*}


Then, we have the following theorem: 
\begin{theorem}
\label{theorem_3}
Given above defined operator $\mathrm{P}_{d_0}$, nonnegative orthant mapping $\mathrm{P}_{+}$, and indicator function $I_\mathcal{V}(\mathbf{v})$. Then the solution $\mathbf{v}_i$ of Eq.(\ref{k}) equals to solving the proximal operator\footnote{Note that there is a factor 1/2 on the quadratic term in the standard form. Here we remove it to simplify the notation.} of the indicator function 
\begin{equation*}
\begin{split}
\mathrm{prox}_{I_\mathcal{V}}(\mathbf{u}_i) & = \argmin_{\mathbf{v}_i} \left\{ || \mathbf{u}_i - \mathbf{v}_i ||_2^2: \mathbf{v}_i \geq 0, ||\mathbf{v}_i||_0 = d_0\right\} \\
 & = \argmin_{\mathbf{v}_i \in \mathcal{V}} \| \mathbf{u}_i - \mathbf{v}_i \|^2,
\end{split}
\end{equation*}
which has an analytical form, i.e., $\mathbf{v}_i = \mathrm{P}_{d_0}(\mathrm{P}_{+}(\mathbf{u}_i))$.
\end{theorem}

\begin{proof}
See Appendix~\ref{app:theorem3}. 
\end{proof} 

Based on \textbf{Theorem 3}, $\mathbf{v}_i$ can be acquired by selecting $d_0$ largest nonnegative entries of $\mathbf{u}_i$ with corresponding indices
$\mathbf{q}_j=\left[ q_{1,j}, \cdots, q_{d_0,j} \right]$. To be specific, for $i=1,\cdots,d$, we set
\begin{equation}
v_{i,j} =
\begin{cases}
u_{i,j} & \text{$i \in \mathbf{q}_j$}, \\
0 & \text{$i \notin \mathbf{q}_j$}.
\end{cases}
\end{equation}

$\textbf{Update P}$: Finally, for the multiplier $\mathbf{P}$, based on the dual optimal condition~\cite{boyd2011distributed}, its updating rule is
\begin{equation}
\mathbf{P} := \mathbf{P} + \mu (\mathbf{Y-V}),
\end{equation}
where $\mu = \min(\rho \mu, \mu_{max})$, with pregiven $\rho$ and $\mu_{max}$.

\textbf{Algorithm 1} summarizes the algorithmic procedure of our proposed alternating direction algorithm for solving Eq.(\ref{a22}). 

After obtaining $\mathbf{X, Y}$, and $\mathbf{E}$, MFC$_0$ can accomplish the following tasks: 
\begin{itemize}
\item \textbf{Subspace Clustering}: We apply normalized cut~\cite{shi2000normalized}\footnote{One can also use K-means clustering on $\mathbf{Y}$. We choose normalized cut here so as to have a fair comparison with SSC and LRR.} to $\mathbf{Y}^T\mathbf{Y}$ to cluster all data samples to their respective subspaces. 
\item \textbf{Data Reconstruction}: We treat the product of $\mathbf{X}$ and $\mathbf{Y}$ as the reconstruction for the data.
\item \textbf{Error Correction}: We regard $\mathbf{E}$ as the error for the data. 
\item \textbf{Representation Learning}: We permute the rows and/or columns of $\mathbf{Y}$ to make it be (appximate) block-diagonal. Then the block $\tilde{\mathbf{Y}}_k$ indicates the representation for subspace $\mathcal{S}_k$.   
\item \textbf{Basis Learning}: We respectively extract $d_0$ columns of $\mathbf{X}$, whose indices correspond to the row numbers of $\tilde{\mathbf{Y}}_k (k=1,\cdots, K$), to form the basis of subspace $\mathcal{S}_k (k=1, \cdots, K)$.  
\end{itemize}

\renewcommand{\algorithmicrequire}{\textbf{Input:}}
\renewcommand{\algorithmicensure}{\textbf{Output:}}

\begin{algorithm}[!htb]
\caption{Solving the problem Eq.(\ref{a22})}
\label{alg:MFC$_0$}
\begin{algorithmic}
\REQUIRE Observed data matrix $\mathbf{Z}$, subspace dimension $d_0$. \\
\ENSURE $\mathbf{X,Y,E}$. \\
    \textbf{Initialize} $\mathbf{X}^{(0)}$, $\mathbf{E}^{(0)}=\mathbf{V}^{(0)}=\mathbf{P}^{(0)}=0$, $\mu^{(0)}=10^{-3}$; \\
    \textbf{Initialize} $\rho=1.2$, $\mu_{max}=10^3$, $\epsilon=10^{-4}$. \\
\WHILE {not converged}
    \STATE
    Update $\mathbf{X,Y,E,V,P}$ using corresponding equations. \\
    \STATE
    Check the convergence condition: \\
    \qquad $\|\mathbf{Z-XY}\|_\infty \leq \epsilon$ or $\|\mathbf{Y-V}\|_\infty \leq \epsilon$.
\ENDWHILE

\end{algorithmic}
\end{algorithm}

\section{Computational Analysis}




The \emph{computational complexity of MFC$_0$ in each iteration} can be summarized as follows:

\myparatight{Calculate $\mathbf{X}$} 
We first perform matrix subtraction of size $m \times n$ and matrix multiplication of size $m \times n$, $n \times d$ with complexities $O(mn)$ and $O(dmn)$, respectively. Then, we compute the SVD of an $m \times d$ matrix and store both the singular vectors---the complexity is $O(m^2d+d^3)$. Finally, we calculate the matrix product of size $m \times d$ and $d \times d$ with complexity $O(d^2m)$ to obtain the orthonormal basis $\mathbf{X}$. The total complexity is $O(m^2d+d^3)+O(d^2m+dmn)$.

\myparatight{Compute $\mathbf{Y}$} 
We just perform matrix multiplication of size $n \times m$ and $m \times d$ and matrix addition of size $m \times d$, with the dominant complexity $O(dmn)$.

\myparatight{Compute $\mathbf{E}$} 
The computation consists of matrix multiplication, matrix substraction, and proximal operator calculation\footnote{Both $\mathbf{E}_{2,1}$ and $\mathbf{E}_1$ have the same complexity $O(mn)$.}. The complexity of each part is $O(dmn)$, $O(mn)$, and $O(mn)$, and the dominant complexity is $O(dmn)$.

\myparatight{Compute $\mathbf{V}$} 
The complexity of matrix addition is $O(dn)$. The nonnegative orthant mapping and sorting of $d$-dim vector respectively need $O(d)$ and $O(d \ln d)$ (in the average case) calculation amount. Thus, the dominant complexity is $O(nd \ln d)$.

As the dimensionality of multi-subspace $d$ is much smaller than the data dimensionality $m$ and sample size $n$, i.e., $m,n >> d$, we conclude that the \emph{dominant complexity} of MFC$_0$ is $O(dm^2+dmn)$.

\begin{table}[t]
\small
\caption{Dominant complexity of SSC, LRR, and MFC$_0$}
\label{compu_complex}
\begin{center}
\begin{tabular}{|c|c|c|c|} 
\hline
Methods &    {SSC}&      {LRR}&      {MFC$_0$}  \\ \hline
Complexity &     $O(n^2m)$&           $O(n^2m+n^3)$&        $O(dm^2+dmn)$  \\ \hline
\end{tabular}
\end{center}
\end{table}

Here, we also discuss the computational complexity of SSC and LRR. \textbf{For SSC}, the computational burden is matrix multiplication, with the complexity $O(n^2m)$. \textbf{For LRR}, the computational burden lies in two parts: matrix multiplication and matrix SVD, the complexity of which are $O(n^2m)$ and $O(n^3)$, respectively. One should note that: (1) the matrix size for matrix multiplication of MFC$_0$  ($m \times n$ and $n \times d$ ) is smaller than that of SSC and LRR ($m \times n$ and $n \times n$); (2) the matrix size for SVD of MFC$_0$ ($m \times d$) is smaller that of LRR ($m \times n$). The dominate complexity of SSC, LRR, and MFC$_0$ are listed in Table~\ref{compu_complex}.

Compared MFC$_0$ with SSC and LRR, we observe that when handling data with high dimensionality, SSC and LRR are relatively faster---the complexity in terms of $m$ is \emph{quadratic} for MFC$_0$, and \emph{linear} for SSC and LRR. Actually, we can leverage random projections, a very 
effective~\cite{wright2009robust} feature extraction method, as a preprocessing to reduce the dimensionality, thus largely decreasing the complexity of MFC$_0$. In contrast, if we deal with large data sample size, MFC$_0$ is more efficient than SSC and LRR---the complexity of SSC and LRR are quadratic and cubic to the sample size $n$, while MFC$_0$ is linear.

In practice, the \emph{overall computational complexity} is determined by not only the complexity of each iteration but also the number of total iterations. Although there is no theoretical guarantees, experimental results (See Tables~\ref{YaleCC}-\ref{USPSCC}) show that MFC$_0$ can converge with 20-40 iterations, while LRR and SSC need 55-70 and 90-110 iterations, respectively. 

\section{Experimental Results}

In this section, we carry out several experiments on both synthetic data and real-world datasets to test the performance of our proposed MFC$_0$ method. 

\myparatight{Compared methods} We compare MFC$_0$ with single subspace learning methods, i.e., PCA, NMF~\cite{lee1999learning}, and Sparse NMF (SNMF)~\cite{peharz2012sparse}; and state-of-the-art multi-subspace learning methods, i.e., SSC and LRR. MFC$_0$ can leverage different regularizations to handle different types of errors. In the experiment, we mainly focus on random corruptions ($||\mathbf{E}||_1$) and sample-specific outliers ($||\mathbf{E}||_{2,1}$). Accordingly, we denote MFC$_0$ as MFC$_0$1 and MFC$_0$2, respectively.

For PCA, we preserve 95 percent of total data variance; For NMF and SNMF, we set their basis number as $K d_0$, the same as MFC$_0$'s. For MFC$_0$, we randomly generate a matrix with size $m \times K d_0 $ to initialize the basis matrix. The hyperparameters of all comparable methods are selected via cross-validation. The stopping criterion for each method is either it reaches the maximal iterations $10^3$ or the objective function values between neighboring iterations have difference less than $10^{-4}$. All experiments are conducted in Matlab R2010b with the platform CPU 3.10 GHz and RAM 16.0 GB.

\myparatight{Evaluation metrics} 
We use clustering accuracy, time, and iterations as the metrics to evaluate the performance of all compared methods.

\subsubsection{Clustering Accuracy} 
The subspace clustering result is evaluated by comparing the estimated label of each data sample with that provided by the ground truth. In the paper, we use clustering accuracy (ACC) to measure the clustering performance. ACC is defined as
\begin{equation}
    ACC = \frac{1}{n} \sum_{i=1}^n \delta(s_i,map(r_i)),
\end{equation}
where $r_i$ and $s_i$ are the estimated label and the ground truth of the $i$-th point; $\delta(x,y)=1$ if $x=y$, and $\delta(x,y)=0$ otherwise. map(x) is the permutation mapping function that maps each label $r_i$ to the equivalent label from the entire data. The best mapping can be efficiently computed by the Kuhn-Munkres algorithm~\cite{plummer1986matching}.

\subsubsection{Time and Iterations}
We use three indices to quantize the computational effciency, i.e., total running time $T(s)$, iterations $I$, and averaged running time per iteration ($T/I$). $T$ indicates the overall effciency; $I$ reflects the convergence rate to reach local solutions, and $T/I$ shows the efficiency per iteration. 

\subsection{Results on Synthetic Data}
\subsubsection{Data reconstruction visualization on 3-dim space}
We visualize data reconstruction in 3-dim ambient space in Figure~\ref{fig1} and compare MFC$_0$ with SSC~\cite{SSC:PAMI2014} and LRR~\cite{LRR:PAMI2014}. The data samples are drawn from three 1-dim independent subspaces, and disturbed by random corruptions and sample-specific outliers, respectively. 
Figure~\ref{fig1} shows that SSC and LRR are unable to remove the errors, which validates our aforementioned argument; While MFC$_0$ fully eliminates them. The reason is that both SSC and LRR use the original contaminated data as the basis, which is unreasonable for data reconstruction. In contrast, MFC$_0$ learns the basis that span the underlying subspace the clean data lie in--Indeed, the learnt basis are $[0, \sqrt{2}/2, \sqrt{2}/2]$, $[\sqrt{2}/2, 0, \sqrt{2}/2]$, and $[\sqrt{2}/2, \sqrt{2}/2, 0]$, respectively; Errors contained in the data are absorbed in the regularization term.

\subsubsection{Results on high-dimensional data}
For high-dimensional data analysis, the synthetic data are generated from $K=5$ independent subspaces with each containing $n_k=100$ samples. All subspaces have the same dimension $d_0=10$ embedded in a $D=100$ dimensional ambient space. The procedure of generating above subspaces is similar to that of~\cite{tangstructure}: the basis $\mathbf{U}_k$ of each subspace are calculated by $\mathbf{U}_{k+1} = \mathbf{TU}_k, 1 \leq k \leq 4$, where $\mathbf{T}=\mathrm{orth(rand(D))} \in \Re^{D \times D}$ is a random orthonormal matrix and $\mathbf{U}_1 \in \Re^{D \times d_0}$ is a random column orthogonal matrix. The data points from each subspace are sampled by $\mathbf{Z}_k=\mathbf{U}_k \mathbf{R}_k$, where elements in $\mathbf{R}_k \in \Re^{d_0 \times n_k}$ are from standard uniform distribution. 

\myparatight{Representation learning and stability} We exhibit MFC$_0$'s ability to learn the representation matrix, i.e., the block-diagonal structure, of data without and with errors, and show its efficiency and stability to obtain the local solution. 


The representation matrix of clean data is shown in Figure $\ref{fig2}\mathrm{(a)}$ and objective function values versus iteration in 10 runs are plotted in Figure $\ref{fig2}\mathrm{(b)}$. From Figure $\ref{fig2}$, we can see that MFC$_0$ can precisely discover the block-diagonal structure when the data samples are clean. Moreover, our proposed first-order optimization algorithm can fast converge (nearly 15 iterations) to a very stable value although using different random initializations. These observations demonstrate that MFC$_0$ is both effective and efficient to analyze the structure of multiple subspaces.

We further consider the cases that data contain errors. For random corruptions and sample-specific outliers with error ratio\footnote{The proportion of data samples are contaminated with errors.} $=$ 0.1, 0.3, and 0.5, the corresponding representation matrices are plotted in Figure $\ref{fig3}$ and Figure $\ref{fig4}$, respectively. We observe that even the data are contaminated with different types of errors to a high level, the representation matrix learnt via MFC$_0$ is still close to be block-diagonal.

\myparatight{Clustering accuracy vs. error ratio} We compare MFC$_0$ with all the other methods in terms of subspace clustering to validate its robustness to resist different types of errors. The performance of all compared methods versus different error ratios, ranging from 0 to 0.8, of random corruptions and sample-specific outliers in the data are depicted in Figure $\ref{fig5}$(a) and $\ref{fig5}$(b). It can be seen that the clustering accuracy of MFC$_0$ is consistently the highest. SSC and LRR come next. Surprisingly, even when the ratio is 0.6, MFC$_0$ can still obtain $>95\%$ accuracy in both cases. Whereas, for LRR and SSC, the accuracies are approximate 70\%, 85\% for random corruptions, and 60\%, 60\% for sample-specific outliers, respectively. The performance of PCA, NMF, and SNMF decrease sharply, which validates that single subspace learning methods fail to work when the error ratio increases.  

\subsubsection{Summary} 
\begin{itemize}
\item MFC$_0$ can accurately reconstruct the data and discover the basis that span each subspace;   

\item MFC$_0$ is capable of obtaining the exact block-diagonal representation matrix when data is clean; 

\item MFC$_0$ shows strong ability to resist errors with a high ratio and is more powerful than state-of-the-art LRR and SSC methods for subspace clustering, not to mention single subspace learning methods; 

\item MFC$_0$ can fast converge to a very stable solution via our proposed first-order optimization algorithm. 
\end{itemize}

\begin{figure}[t]
\begin{center}
\subfigure[]{\label{fig2:a}\includegraphics[width=0.48\linewidth]{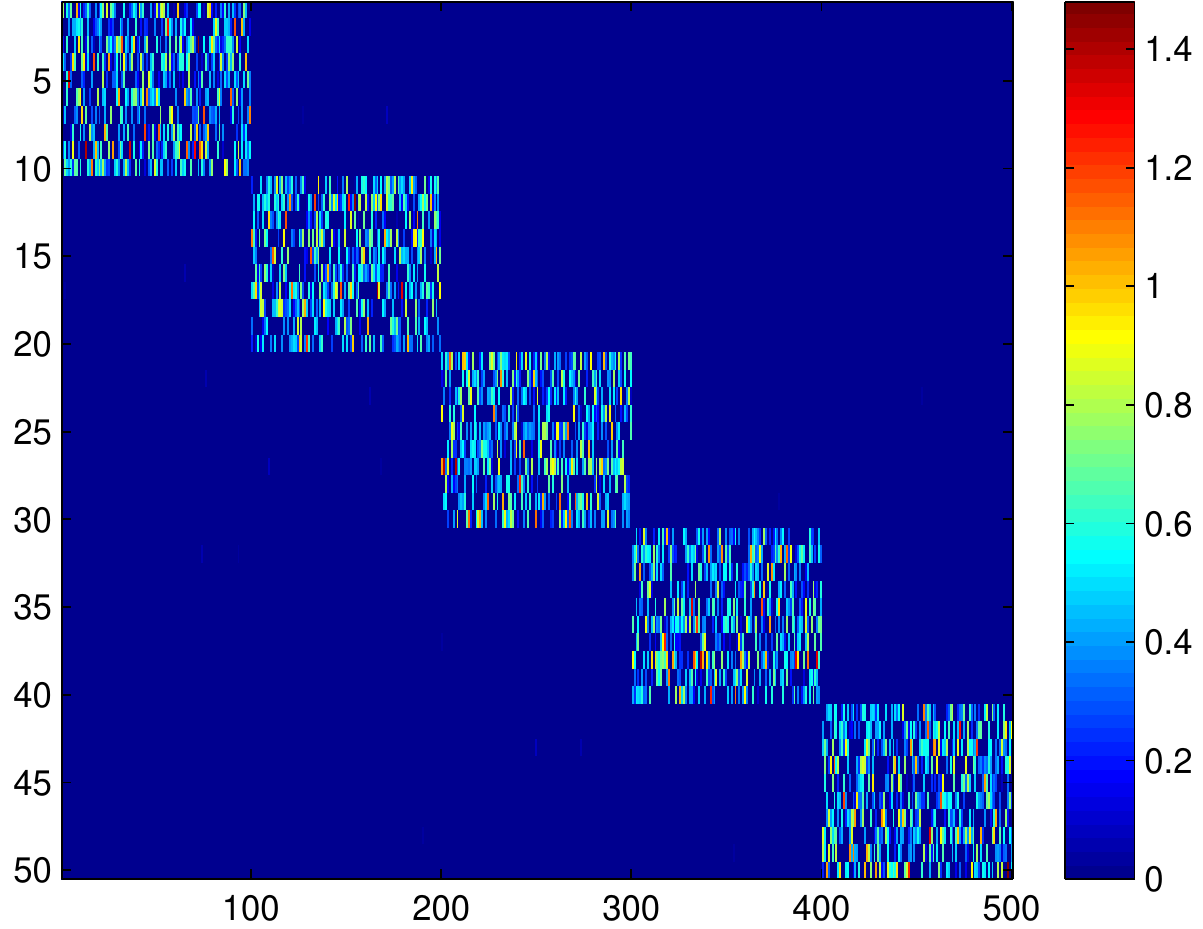}}
\subfigure[]{\label{fig2:b}\includegraphics[width=0.48\linewidth]{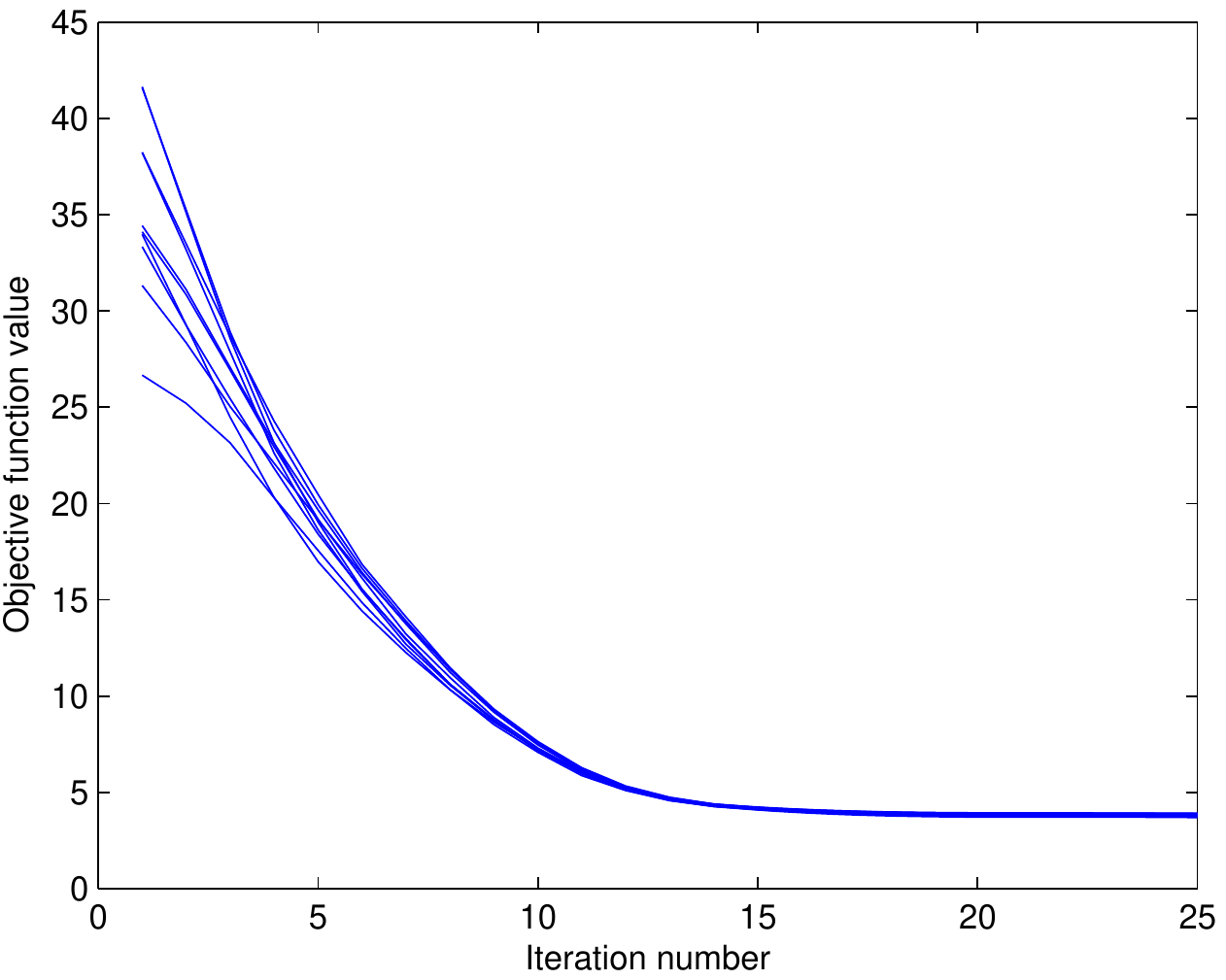}}
\end{center}
\caption{(a) Representation matrix learnt on clean data. (b) Objective function values vs. iteration.}
\label{fig2}
\end{figure}

\begin{figure}[t]
\begin{center}
\subfigure[ratio=0.1]{\label{fig3:b}\includegraphics[width=0.32\linewidth]{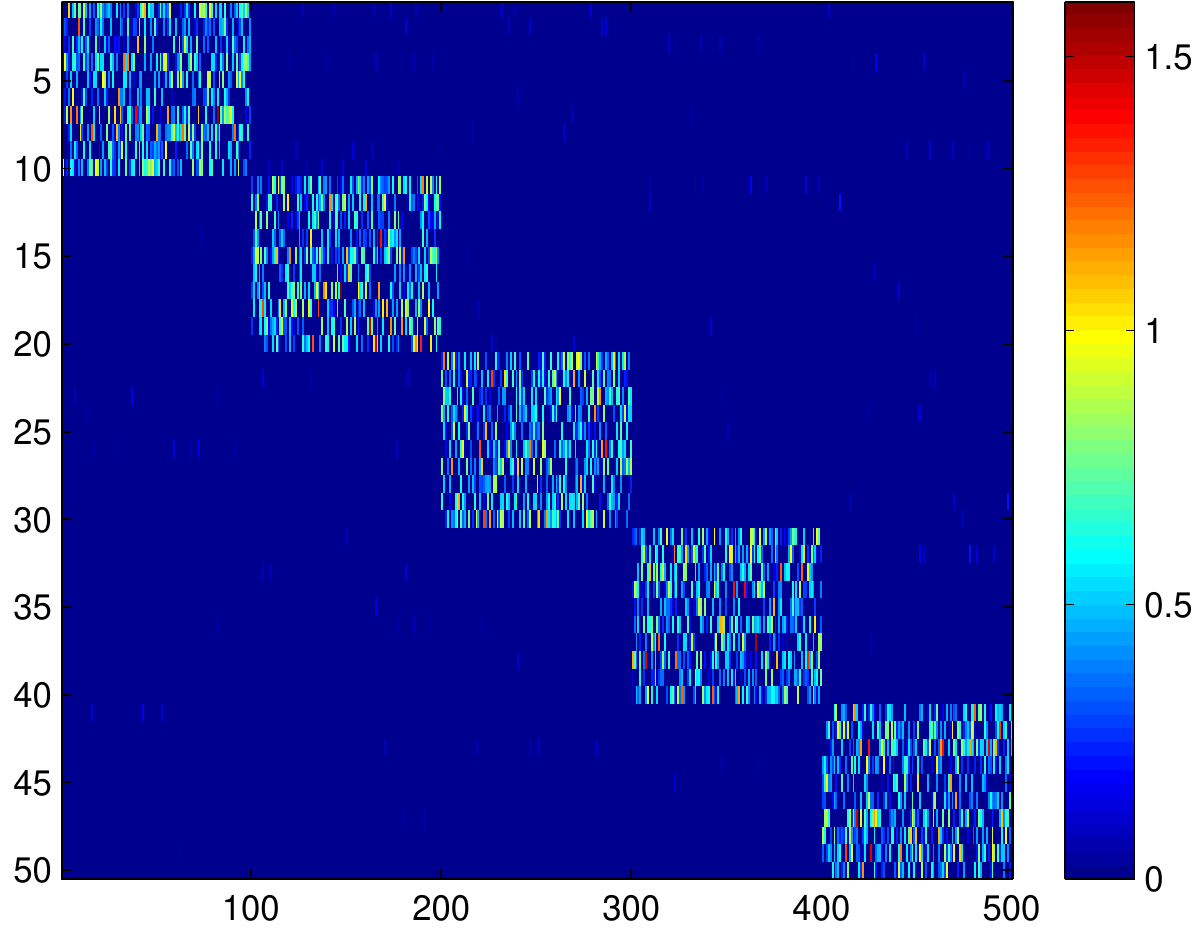}}
\subfigure[ratio=0.3]{\label{fig3:c}\includegraphics[width=0.32\linewidth]{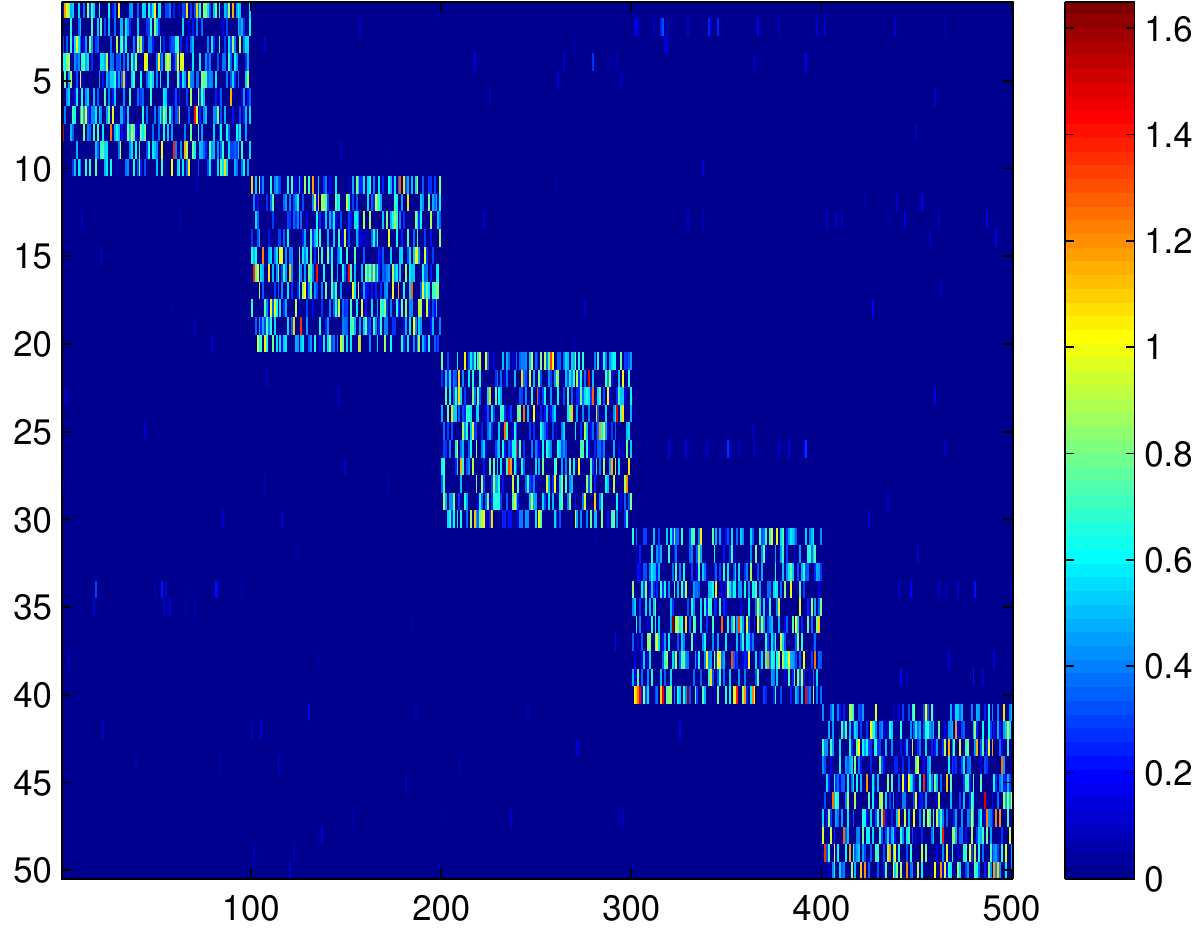}}
\subfigure[ratio=0.5]{\label{fig3:d}\includegraphics[width=0.32\linewidth]{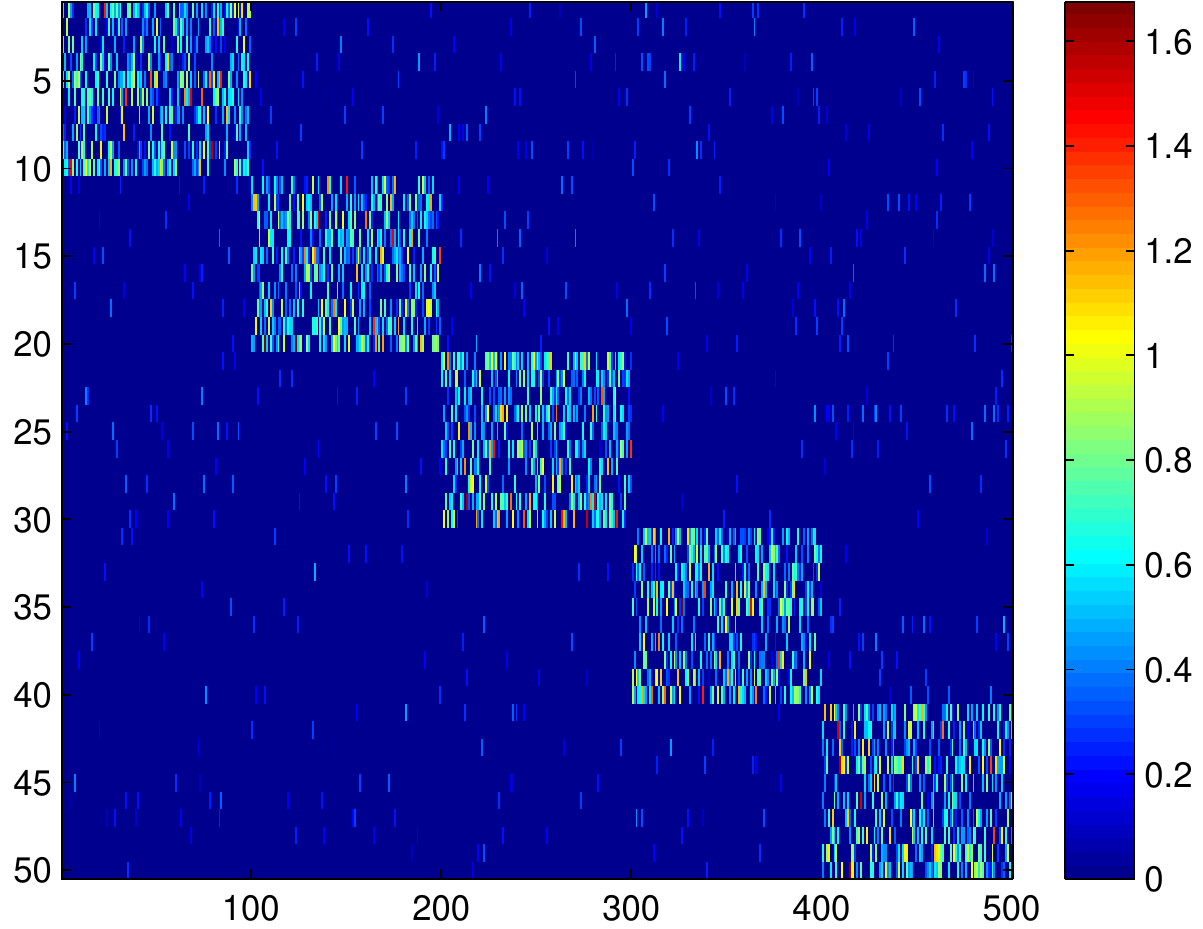}}
\end{center}
\caption{Representation matrix learnt on data contaminated by different ratios of random corruptions.}
\label{fig3}
\end{figure}

\begin{figure}[t]
\begin{center}
\subfigure[ratio=0.1]{\label{fig4:b}\includegraphics[width=0.32\linewidth]{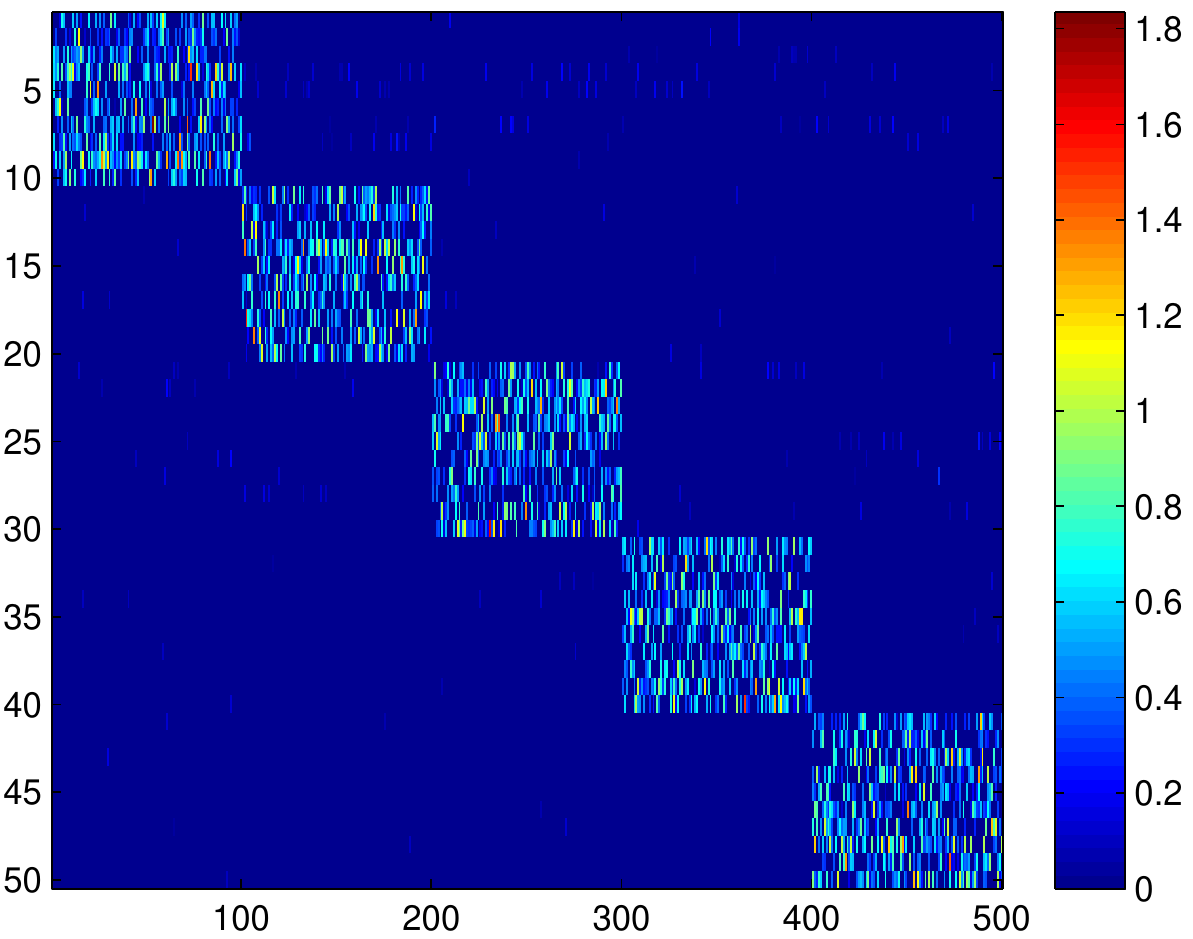}}
\subfigure[ratio=0.3]{\label{fig4:c}\includegraphics[width=0.32\linewidth]{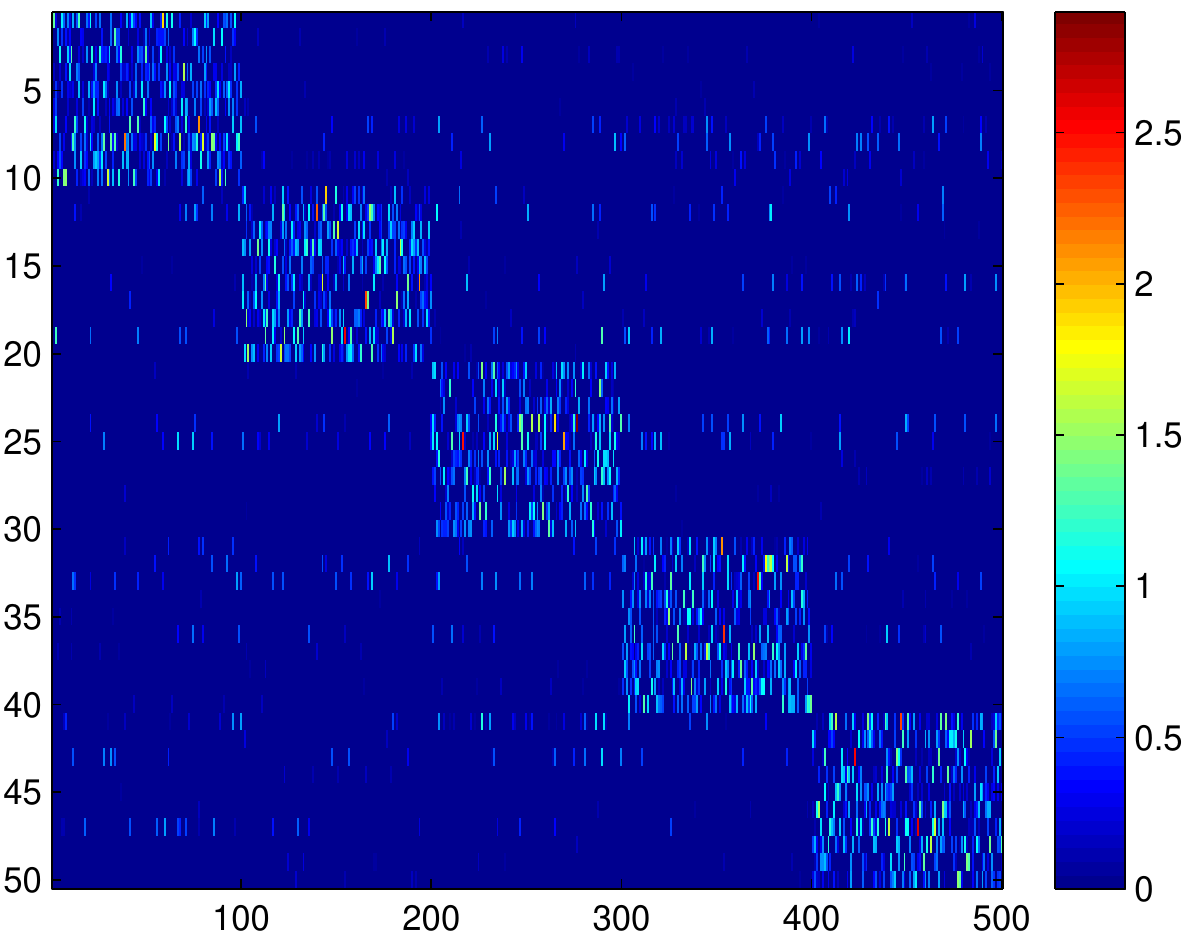}}
\subfigure[ratio=0.5]{\label{fig4:d}\includegraphics[width=0.32\linewidth]{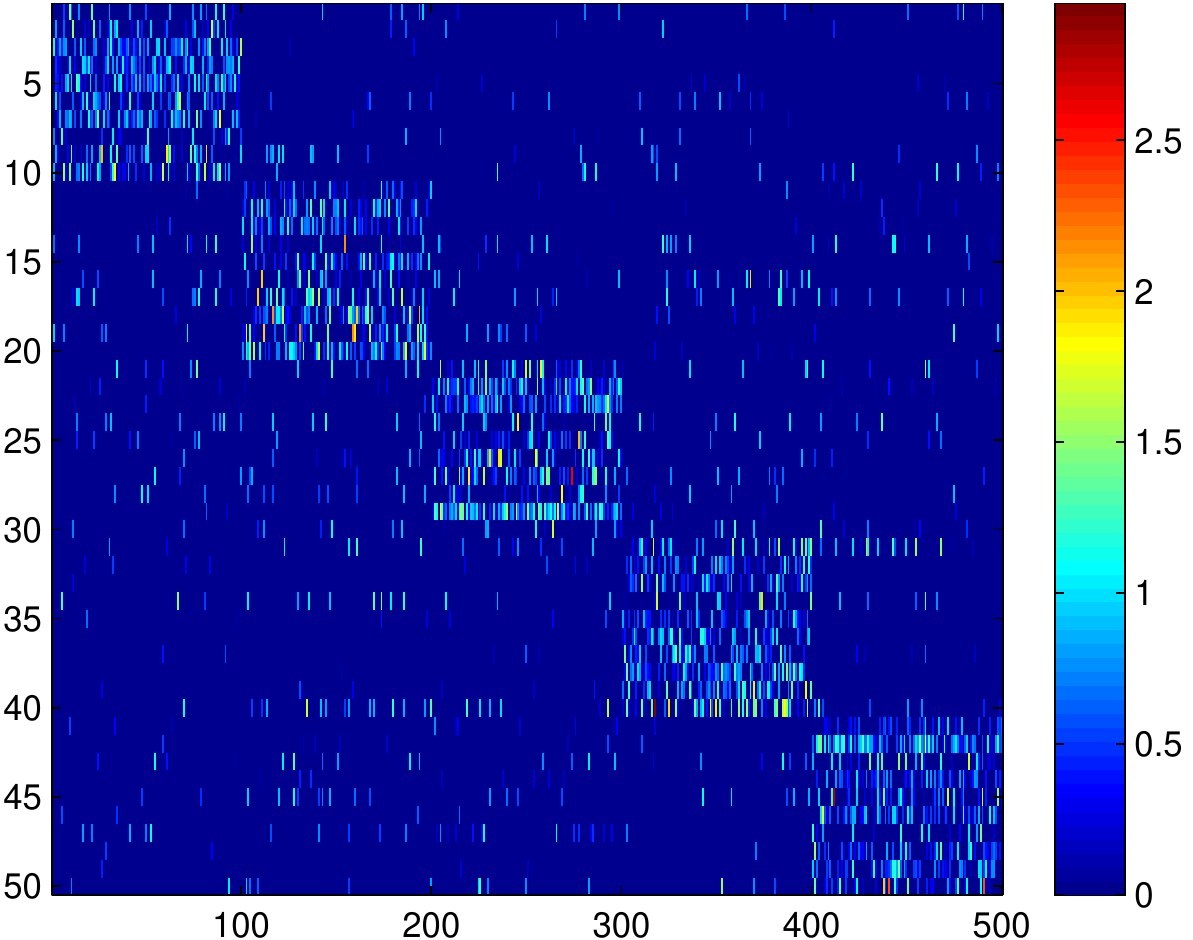}}
\end{center}
\caption{Representation matrix learnt on data contaminated by different ratios of sample-specific outliers.}
\label{fig4}
\end{figure}

\begin{figure}[!hbt]
\begin{center}
\subfigure[]{\label{fig5:a}\includegraphics[width=0.48\linewidth]{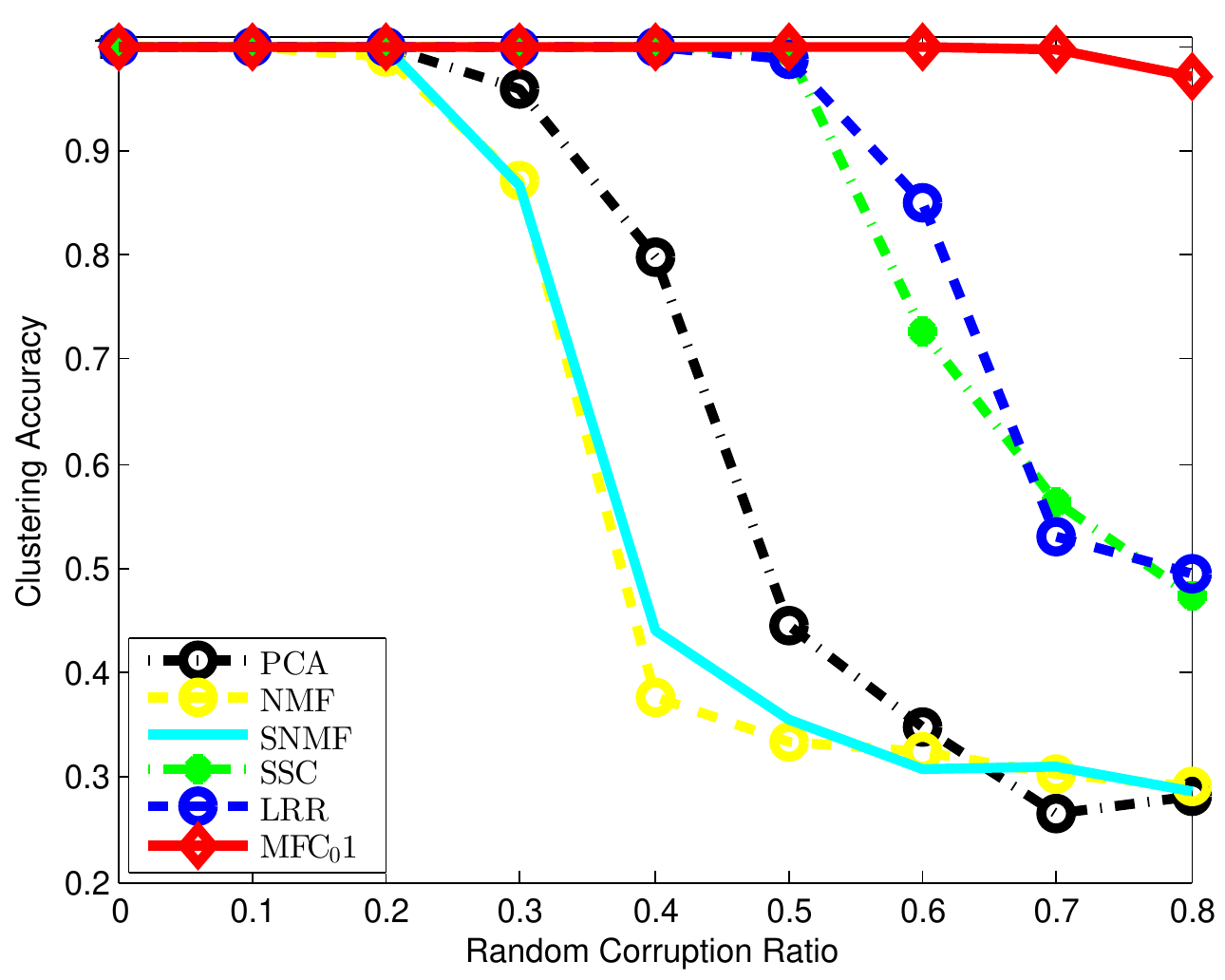}}
\subfigure[]{\label{fig5:b}\includegraphics[width=0.48\linewidth]{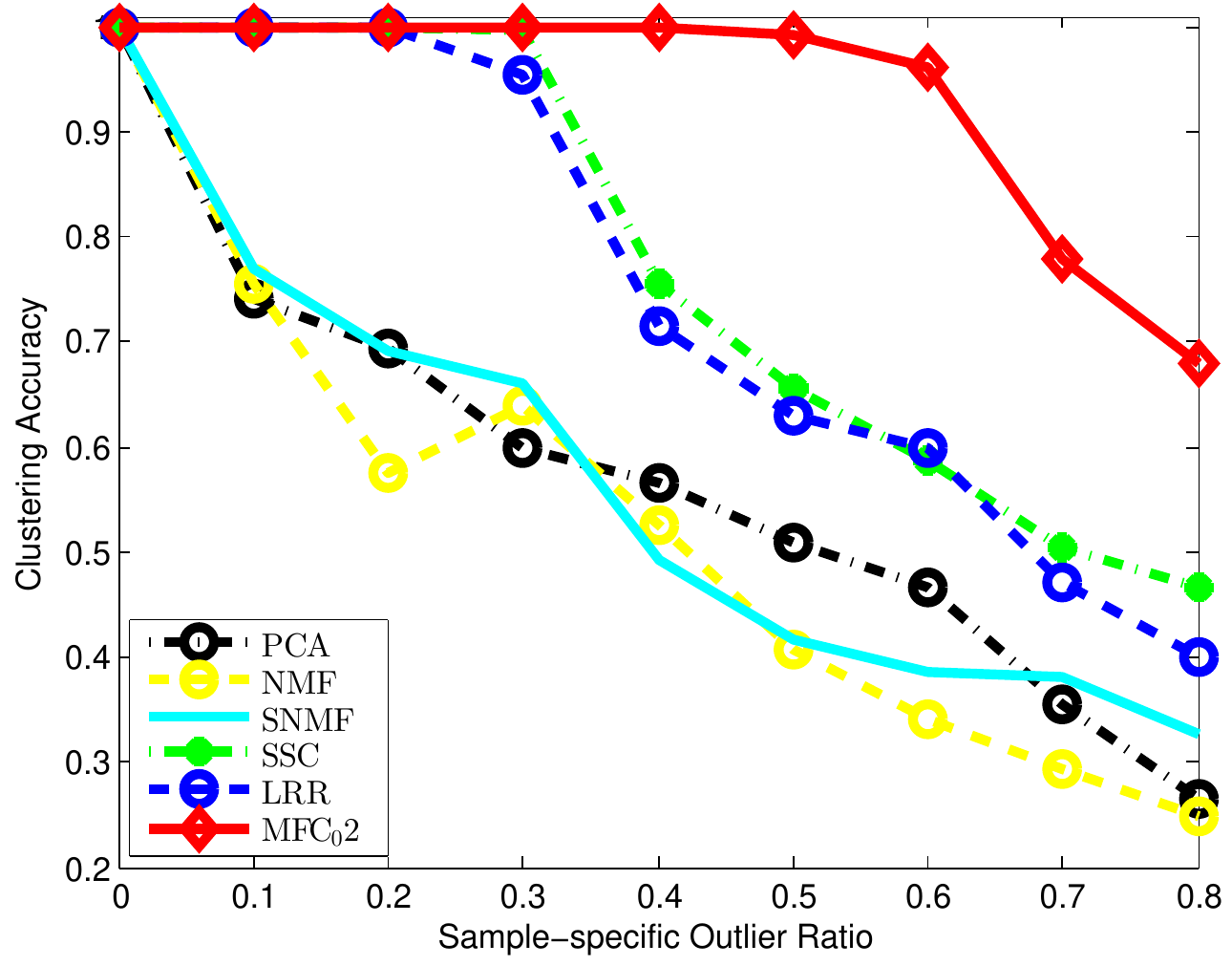}}
\end{center}
\caption{Clustering accuracy vs. ratio of (a) random corruptions and (b) sample-specific outliers in the data.}
\label{fig5}
\end{figure}

\subsection{Results on Real-World Datasets}

In this part, we pay close attention to the tasks: subspace clustering, as well as associated computational complexity; subspace basis learning and subspace recovery. Experiments are carried out on three real-world face datasets: AR\footnote{\url{http://www2.ece.ohio-state.edu/~aleix/ARdatabase.html}}, Yale\footnote{\url{http://vision.ucsd.edu/content/yale-face-database}}, and Extended Yale B (EYaleB)\footnote{\url{http://vision.ucsd.edu/~leekc/ExtYaleDatabase/ExtYaleB.html}} and the USPS handwritten digit dataset\footnote{\url{http://www.csie.ntu.edu.tw/~cjlin/libsvmtools/datasets/multiclass.html\#usps}}. 

\subsubsection{Dataset Description}

AR contains over 3,000 images of 126 individuals (70 men and 56 women). They are taken at two different occasions with different facial expressions, illumination conditions, and occlusions (sun glasses and scarf) under strictly controlled conditions. Raw images are of 768x576 pixels and 24 bits of depth. In the experiment, images are resized to 48x48 pixels. 

Yale contains 165 grayscale images of 15 individuals. There are 11 images per subject, one per different facial expression or configuration: center-light, w/glasses, happy, left-light, w/no glasses, normal, right-light, sad, sleepy, surprised, and wink. We simply use the cropped images with size 32x32 pixels. 

EYaleB consists of 2414 images of 38 subjects under 9 poses and 64 illumination conditions. The subset used in the experiment has 20 individuals and around 64 near frontal images under different illuminations per individual. All images are cropped to 32x32 pixels.

USPS contains 9198 handwritten digit images (from 0 to 9) and the image size is 16x16. In the experiment, we choose 100 images of each digit, thus 1000 images in total. 

Some sample images from the four datasets are shown in Figure $\ref{fig6}$. Statistics of the selected subsets are listed in Table~\ref{Stats}.

\begin{table}[t]
\small
\caption{Statistics of selected subsets of four datasets}
\label{Stats}
\begin{center}
\begin{tabular}{|c|c|c|c|c|} 
\hline
&    {AR}&      {Yale}&      {EYaleB}&   {USPS}   \\
\hline
K &     3&           11&        20&          10   \\ \hline
$n_k$ &   15&          15&        64&          100     \\ \hline
m &     48x48&       32x32&     32x32&       16x16     \\
\hline
\end{tabular}
\end{center}
\footnotesize
\end{table}

\begin{figure}[t]
\begin{center}
\subfigure[]{\label{fig6:a}\includegraphics[width=0.48\linewidth]{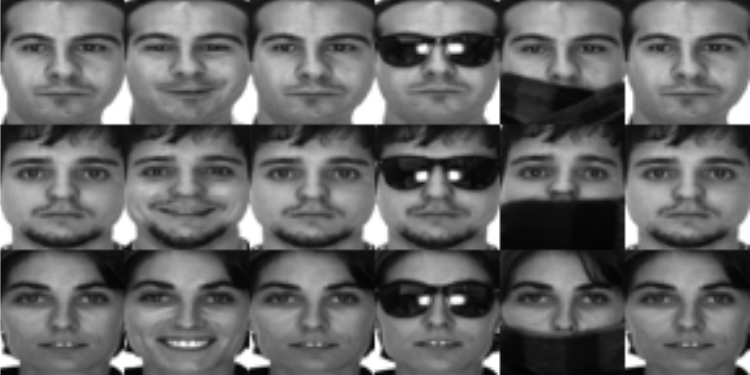}}   
\subfigure[]{\label{fig6:d}\includegraphics[width=0.48\linewidth]{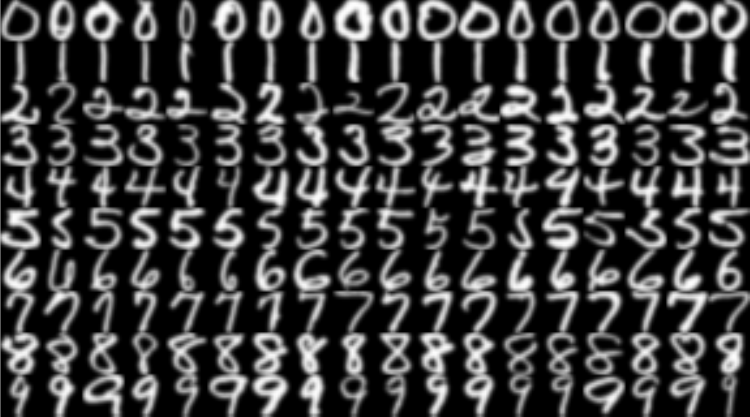}} 
\subfigure[]{\label{fig6:b}\includegraphics[width=0.48\linewidth]{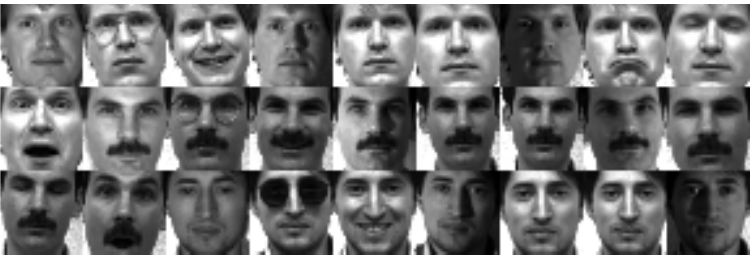}}  
\subfigure[]{\label{fig6:c}\includegraphics[width=0.48\linewidth]{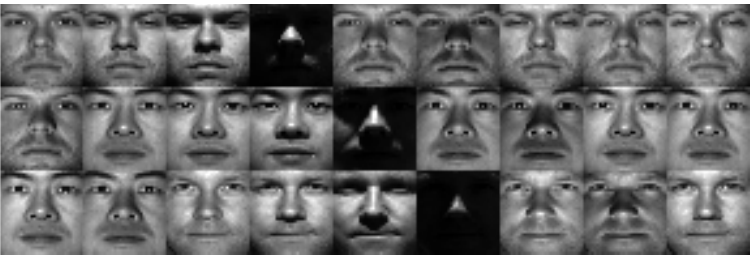}}  
\end{center}
\caption{Sample images from (a) AR, (b) USPS, (c) Yale, and (d) EYaleB.}
\label{fig6}
\end{figure}

\subsubsection{Subspace Dimension Estimation}

For MFC$_0$, we need to know the subspace dimension of each subject in advance. A simple yet effective way is to calculate the number of nonzero singular values of each subject. Figure $\ref{fig7}$ depicts the singular values of several randomly chosen subjects in the four datasets. The subspace dimensions $d_0$ for AR, Yale, EYaleB, and USPS are approximate 12, 10, 10, and 12 respectively. Note that for real-world datasets, the singular values often gradually close to zero, validating that the images are contaminated with errors. The higher dimension of AR reflects its more complicated face distribution than Yale and EYaleB. 
It is necessary to point out that, we cannot distinguish which type of error contained in the images, therefore we report  

\begin{figure}[t]
\begin{center}
\subfigure[]{\label{fig7:a}\includegraphics[width=0.48\linewidth]{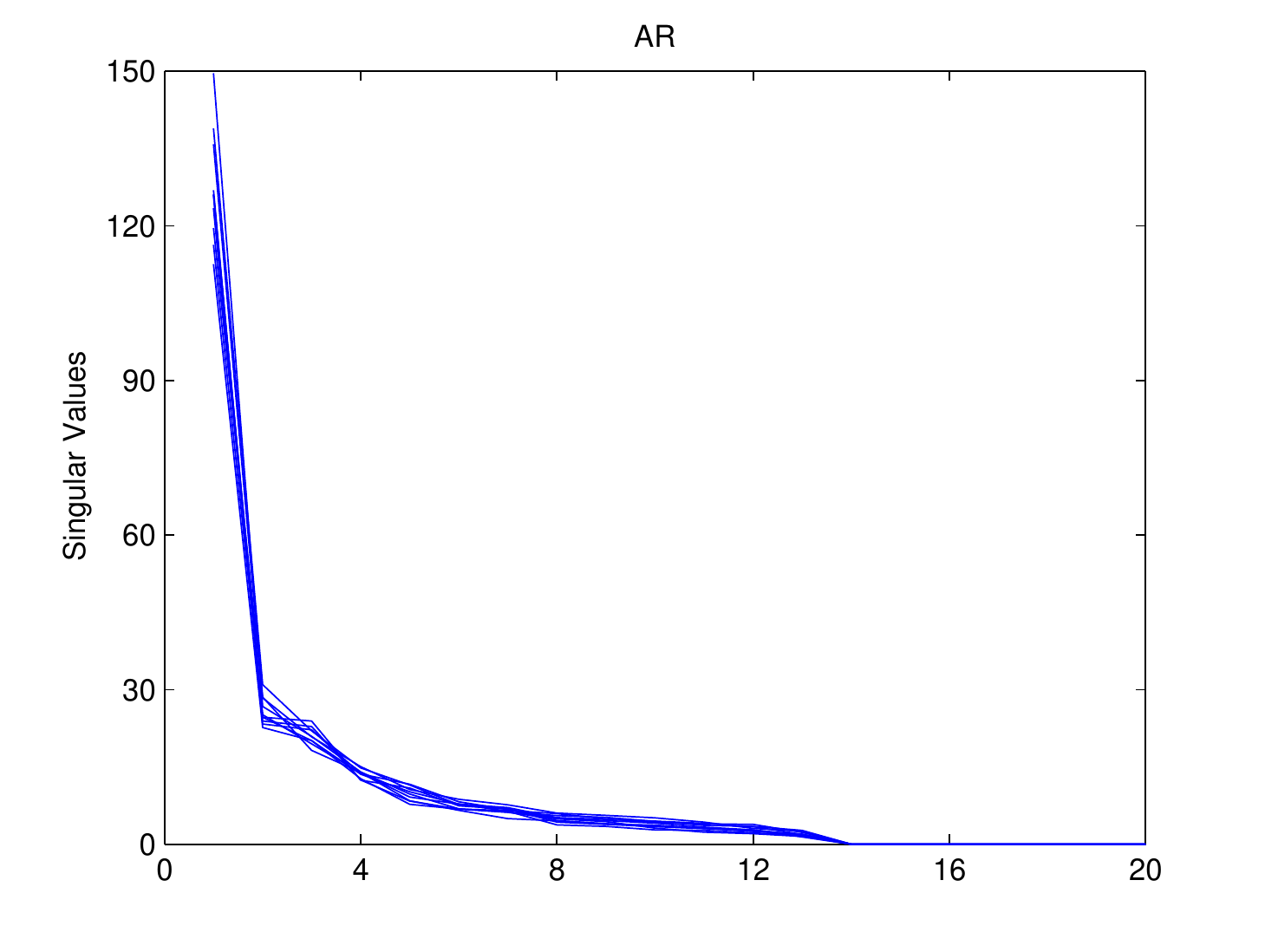}}  
\subfigure[]{\label{fig7:b}\includegraphics[width=0.48\linewidth]{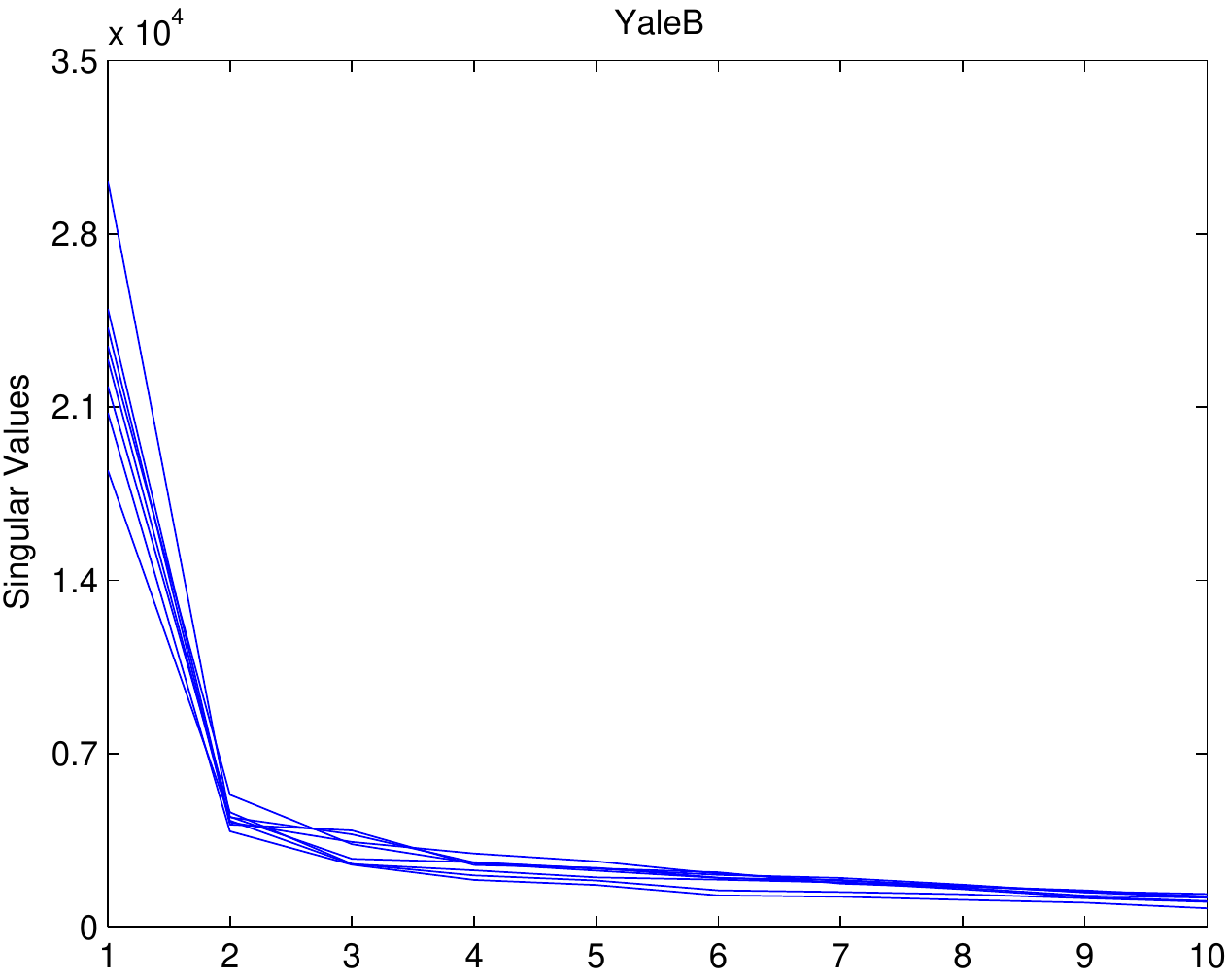}}  
\subfigure[]{\label{fig7:c}\includegraphics[width=0.48\linewidth]{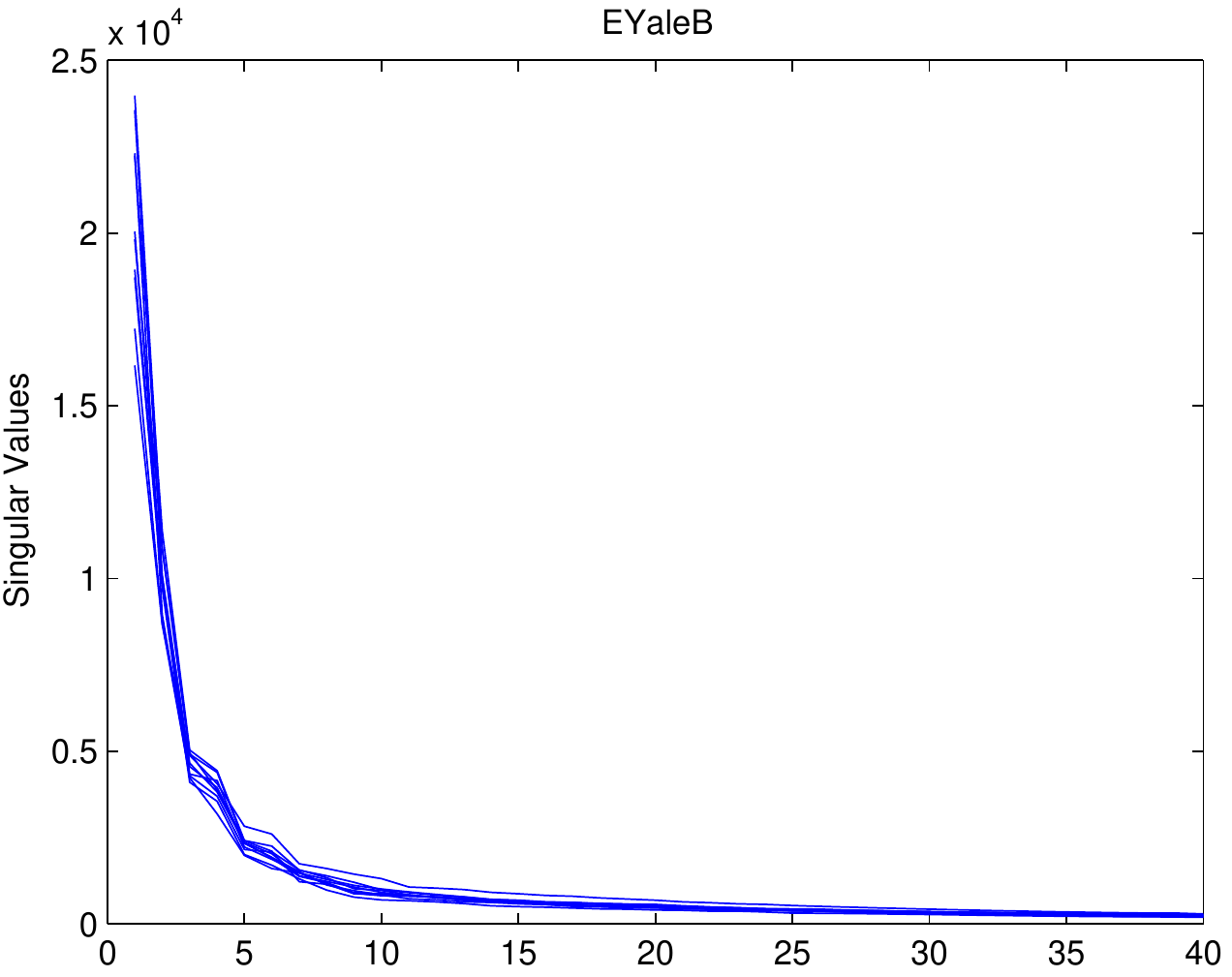}}  
\subfigure[]{\label{fig7:d}\includegraphics[width=0.48\linewidth]{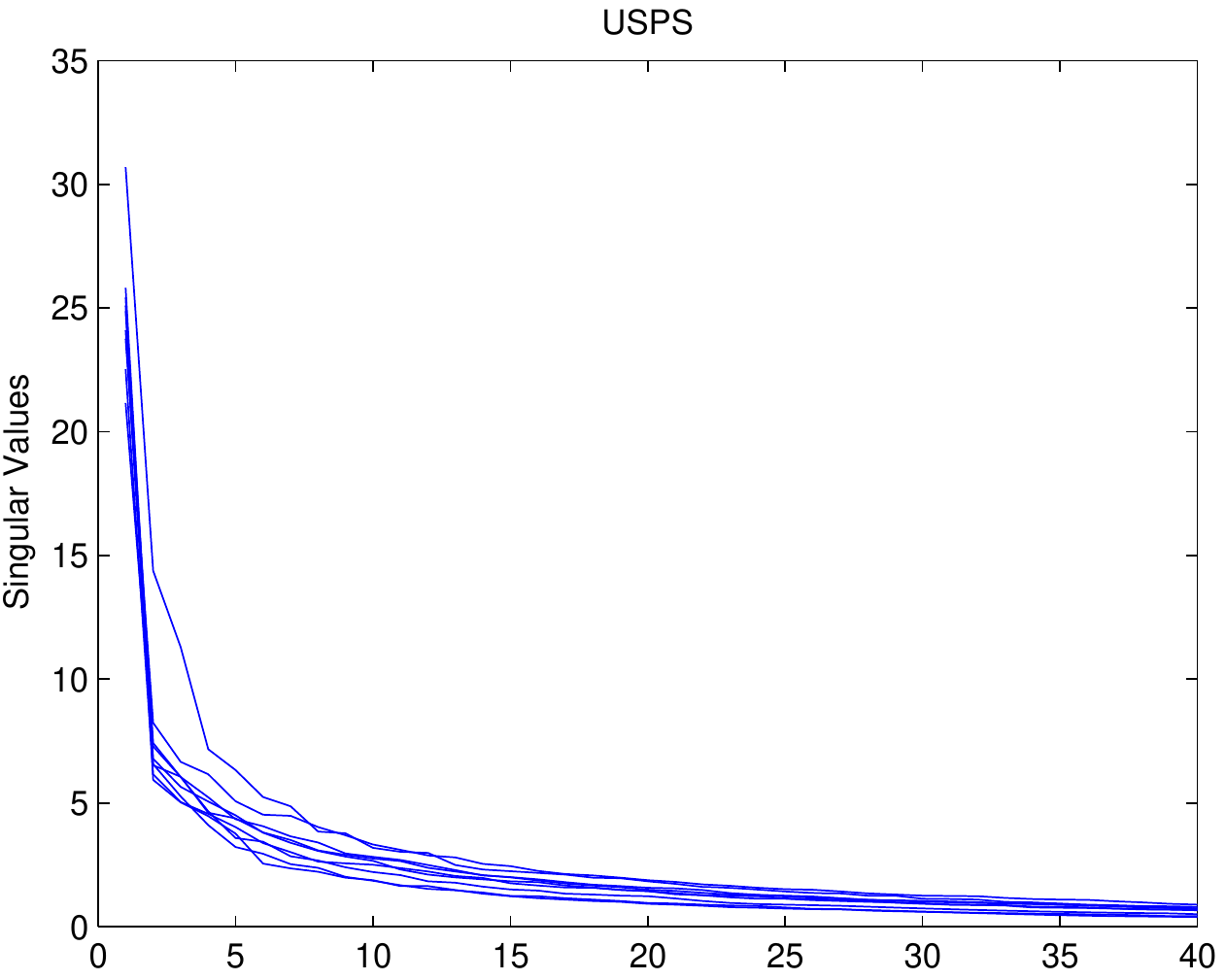}}  
\end{center}
\caption{Singular values of several randomly chosen subjects in (a) AR, (b) Yale, (c) EYaleB, and (d) USPS. }
\label{fig7}
\end{figure}

\subsubsection{Subspace Clustering}

In this part, we perform subspace clustering on real-world datasets: Given face images or handwritten digits of multiple categories, our task is to group them according to their respective subspace. Our purpose is to verify that it is more advantageous to (1) model face distribution using multiple subspaces than using the single one; (2) learn the basis rather than leverage the data themseleves as the basis. 


Tables~\ref{Yale} and~\ref{EYaleB} show the clustering performance on Yale and EYaleB, respectively. To study the effect of category number, we choose $K$ that is ranging from 2 to 11 on Yale and from 2 to 20 with an interval 2 on EYaleB.
From the results, we have the following observations: 

(1) \textbf{As a whole, MFC$_02$ performs the best on both datasets}. By learning each subject the respective subspace, together with a direct column $l_0$-norm constraint and an error correction term, MFC$_02$ is discriminative to separate face images from the subjects. Essentially, this is the key factor that makes it outperform single subspace analysis methods. Moreover, using a column $l_0$-norm constraint to directly control the sparsity, MFC$_02$ is superior to the indirect $l_1$ penalty and avoids tuning the sparse hyperparameter; 

(2) \textbf{SSC and LRR can also achieve promising results} as they are multi-subspace learning methods with error correction. However, that they use contaminated face images as the basis is problematic and irrational, the accuracy gap between them and MFC$_02$ validates this argument; 

(3) \textbf{Errors contained in face images are more likely to be sample-specific outliers}, as on both Yale and EYaleB, the result of MFC$_02$ is better than that of MFC$_01$. A corroborative evidence is that LRR, also using the $\mathbf{E}_{2,1}$ term, is slightly superior to SSC, the $\mathbf{E}_{1}$ term; 

(4) \textbf{Sparse constraint for clustering is important}. The better results of SNMF than that of NMF in both datasets verify the importance using sparse constraint;

(5) \textbf{Modeling face images distribution using multi-subspace} Since PCA, NMF, and SNMF are single subspace learning methods, so their performance are worse than SSC, LRR, MFC$_01$, MFC$_02$. When the sample size of each subject is small in Yale, the structure of multi-subspace is not so obvious and the accuracy gap is not remarkable. However, when the sample size of each subject in EYaleB is sufficient, all of them obtain poor results on the whole. Even worse, they fail to work when $K$ increases to a certain value, say 10. While for remaining methods, they still work well.

We also implement subspace clustering on USPS with $K$ from 2 to 10. Our aim is to see whether handwritten digits are distributed in ``independent subspaces", as no pervious work mentions such assumption. 

\textbf{Subspaces of handwritten digits are probably not independent}. From the results listed in Table~\ref{USPS}, we notice that all PCA, NMF, and SNMF have promising results, even comparable with SSC and LRR in some $K$s. Still, our MFC$_01$ performs the best, but the superiority is not so obvious as that in EYaleB. One possible reason is that the subspaces of handwritten digits are not independent in essence. In this sense, the distribution of handwritten digits modeled by disjoint subspaces (or submanifolds) may be more reasonable.

\begin{table}[t]
\caption{Clustering performance on Yale}
\label{Yale}
\begin{center}
\begin{tabular}{|c|c|c|c|c|c|c|c|}
\hline
{K}&        PCA&        {NMF}&      {SNMF}&     {SSC}&     {LRR}&       {MFC$_01$}&    {MFC$_02$} \\
\hline
2 &         0.76&        0.82&       0.82&      0.87&    \textbf{0.95}& \textbf{0.95}&  \textbf{0.95}  \\ \hline
3 &         0.61&        0.70&       0.69&      0.64&    0.68&        0.70&    \textbf{0.74}  \\ \hline
4 &         0.63&      0.64&       0.68&      0.61&    0.64&        0.66&   \textbf{0.76} \\ \hline
5 &         0.53&      0.64&       0.67&      0.67&    0.67&        0.65&   \textbf{0.69} \\ \hline
6 &         0.56&      0.53&       0.58&      0.65&    0.58&        0.63&   \textbf{0.68} \\ \hline
7 &         0.53&      0.55&       0.56&      0.62&    0.57&        0.58&   \textbf{0.61} \\ \hline
8 &         0.55&      0.49&       0.57&      0.60&  \textbf{0.63}& 0.56&       0.60 \\ \hline
9 &         0.43&      0.38&       0.49&      0.55&    0.57&        0.51&   \textbf{0.58} \\ \hline
10 &        0.45&      0.47&  \textbf{0.54}&  \textbf{0.54}&   \textbf{0.54}& 0.51&  \textbf{0.54} \\ \hline
11 &        0.54&      0.49&       0.54&      0.57&    0.57&        0.50&   \textbf{0.59} \\ \hline \hline
Avg.&       0.56&        0.57&       0.61&      0.63&      0.64&        0.63&   \textbf{0.66} \\ \hline
\end{tabular}
\end{center}
\footnotesize
\end{table}

\begin{table}[t]
\caption{Clustering performance on EYaleB}
\label{EYaleB}
\begin{center}
\begin{tabular}{|c|c|c|c|c|c|c|c|}
\hline
{K}&        {PCA}&     {NMF}&    {SNMF}&    {SSC}&     {LRR}&          {MFC$_01$}&         {MFC$_02$} \\
\hline
2 &         0.51&     0.80& \textbf{1.00}&0.98&    \textbf{1.00}&    \textbf{1.00}&       \textbf{1.00}  \\ 
\hline
4 &         0.35&     0.59&    0.63&      0.88&    0.86&             0.85&               \textbf{0.89}  \\ \hline
6 &         0.30&   0.72&    0.57&      0.84&    0.82&             0.84&               \textbf{0.86} \\ \hline
8 &         0.20&   0.37&    0.48&      0.76&    0.74&             0.69&               \textbf{0.76} \\ \hline
10 &        0.23&   0.45&    0.41&      0.74&    0.71&             0.68&               \textbf{0.75} \\ \hline
12 &        0.22&   0.37&    0.41&      0.67&    \textbf{0.70}&    \textbf{0.70}&      \textbf{0.70} \\ \hline
14 &        0.18&   0.36&    0.39&      0.64&    0.68&             \textbf{0.77}&          0.72 \\ \hline
16 &        0.18&   0.37&    0.39&      0.63&    0.67&             0.72&               \textbf{0.73} \\ \hline
18 &        0.15&   0.35&    0.35&      0.63&    0.68&             \textbf{0.74}&      \textbf{0.74} \\ \hline
20 &        0.15&   0.37&    0.37&      0.60&    0.66&             \textbf{0.72}&      \textbf{0.72} \\ \hline \hline
Avg.&       0.25&     0.47&     0.49&      0.74&    0.75&               0.77&            \textbf{0.79} \\ \hline
\end{tabular}
\end{center}
\footnotesize
\end{table}

\begin{table}[t]
\caption{Clustering performance on USPS}
\label{USPS}
\begin{center}
\begin{tabular}{|c|c|c|c|c|c|c|c|} \hline
{K}&        PCA&     {NMF}&     {SNMF}&     {SSC}&     {LRR}&          {MFC$_01$}&         {MFC$_02$} \\
\hline
2 &         0.92&     0.98&      0.96&      0.99&      0.99&             \textbf{1.00}&       \textbf{1.00}  \\ \hline
3 &         0.80&     0.91&      0.86&      0.95&    0.91&             \textbf{0.94}&       \textbf{0.94}  \\ \hline
4 &         0.63&   0.76&      0.78&      0.85&    0.89&             \textbf{0.93}&             0.87 \\ \hline
5 &         0.67&   0.64&      0.74&      0.80&    0.79&                    0.87&      \textbf{0.88} \\ \hline
6 &         0.57&   0.56&      0.67& \textbf{0.80}&  0.77&                  0.78&             0.76 \\ \hline
7 &         0.56&   0.55&      0.58&      0.68&    0.65&             \textbf{0.74}&             0.64 \\ \hline
8 &         0.56&   0.57&      0.57&      0.54&    0.62&                    0.66&            \textbf{0.69} \\ \hline
9 &         0.54&   0.52&      0.51&      0.46&    0.61&             \textbf{0.68}&               0.63 \\ \hline
10 &        0.54&   0.47&      0.48&      0.53&    0.47&             \textbf{0.64}&               0.58 \\ \hline \hline
Avg.&       0.64&     0.66&      0.68&      0.73&      0.74&              \textbf{0.80}&             0.78 \\ \hline
\end{tabular}
\end{center}
\footnotesize
\end{table}

\subsubsection{Computational Complexity}

Another advantage of MFC$_0$ against LRR and SSC is its efficient computation in terms of sample size. Therefore, our aim in this part is to validate that when increasing the sample size MFC$_0$ is more efficient than SSC and LRR. As mentioned, we quantize the efficiency of each method using total running time ($T(s)$), iterations ($I$), and averaged running time per iteration ($T/I$). 

The results on Yale, EYaleB, and USPS are listed in Tables~\ref{YaleCC},~\ref{EYaleBCC},~\ref{USPSCC}, respectively. From Table~\ref{YaleCC}, we see that 

(1) \textbf{Running time per iteration $T/I$ of MFC$_01$ or MFC$_02$ is bigger than that of SSC and LRR on dataset with small sample size}. The averaged $T/I$ of MFC$_01$ or MFC$_02$ is approximately 5 and 3 times of SSC and LRR. It is because when the sample size of each subject $n_k=15$ is small in Yale, the dimensionality of data samples ($m=32 \times 32$) dominates the computational complexity in each iteration. Remembering that the complexity of MFC$_0$ is quadratic with respect to $m$, while of LRR and SSC is linear. 

(2) \textbf{Iterations $I$ of MFC$_0$ is much smaller than those of SSC and LRR}. Despite of higher $T/I$, the averaged iteration $I$ of MFC$_01$ or MFC$_02$ is 24 or 23, much smaller than those of SSC and LRR, which are 93 and 57, respectively. Such observation demonstrats MFC$_0$'s faster convergence rate than LRR and SSC. Ultimately, the total running time $T$ of these methods are similar.

(3) \textbf{Increasing sample size makes MFC$_0$ more advantageous}. When the sample size of each subject increases to 64 in EYaleB, the results in Table~\ref{EYaleBCC} manifest that $T/I$ values of SSC and LRR are larger than that of MFC$_01$ or MFC$_02$ on average. It is due to the fact that the computational complexity of MFC$_0$ is linear with respect to $n$, while quadratic and cubic of LRR and SSC. Similar to the results on Yale, compared with SSC (110) and LRR (65), the averaged values $I$ of MFC$_01$ (31) and MFC$_02$ (30) are much smaller. Consequently, the total running time $T$ of MFC$_01$ and MFC$_02$ are just 1/4, 1/5.7 and 1/6, 1/8.5 of SSC and LRR. To further highlighting the efficiency of MFC0$_0$, we add samples to 100 in USPS and related results are displayed in Table~\ref{USPSCC}. We observe that when the category number is 10, it takes SSC and LRR almost 94 and 255 seconds to implement the clustering task, while both MFC$_01$ and MFC$_02$ consume less than 10 seconds. Again, the less $I$ of MFC$_01$ or MFC$_02$ than SSC and LRR suggests they can reach the stable solution more quickly.


\begin{table*}[t]
\begin{center}
 \caption{Computational complexity on Yale}
\centering
\label{YaleCC}
\begin{tabular}{|c|c|c|c|c|c|c|c|c|c|c|c|c|}
\hline &\multicolumn{3}{|c|}{SSC} & \multicolumn{3}{|c|}{LRR} & \multicolumn{3}{|c|}{MFC$_01$} & \multicolumn{3}{|c|}{MFC$_02$}\\
\hline
K   &T(s)  &I   &T/I   &T(s)  &I  &T/I    &T(s)    &I  &T/I     &T(s)   &I  &T/I    \\  \hline \rule{0pt}{6pt}
2 &0.120 &92  &\textbf{0.001} &\textbf{0.105} &57 &0.002  &0.191   &\textbf{22} &0.009   &0.218 &24 &0.009  \\  \hline \rule{0pt}{7pt}
3 &0.196 &99  &\textbf{0.002} &\textbf{0.167} &55 &0.003  &0.281   &\textbf{22} &0.013   &0.316 &\textbf{22}  &0.014   \\ \hline \rule{0pt}{7pt}
4 &0.303 &91  &\textbf{0.003} &\textbf{0.286} &56 &0.005  &0.416   &\textbf{22} &0.019   &0.446 &\textbf{22}  &0.020  \\  \hline \rule{0pt}{7pt}
5 &0.445 &96  &\textbf{0.005} &\textbf{0.396} &57 &0.007  &0.592   &\textbf{22} &0.027   &0.537 &\textbf{22} &0.024   \\  \hline \rule{0pt}{7pt}
6 &\textbf{0.483} &89  &\textbf{0.005} &0.551 &57 &0.010  &0.727   &\textbf{23} &0.032   &0.747 &\textbf{23}  &0.033   \\  \hline \rule{0pt}{7pt}
7 &\textbf{0.643} &89  &\textbf{0.007} &0.718 &57 &0.013  &0.905   &\textbf{23} &0.039   &0.923 &\textbf{23}  &0.040   \\  \hline \rule{0pt}{7pt}
8 &0.830 &90  &\textbf{0.009} &\textbf{0.773} &57 &0.014  &1.149   &27 &0.043   &1.118  &\textbf{24}  &0.047   \\  \hline \rule{0pt}{7pt}
9   &\textbf{1.007} &92  &\textbf{0.011} &1.051 &57 &0.018  &1.339   &25 &0.054   &1.503  &\textbf{24} &0.063   \\  \hline \rule{0pt}{7pt}
10  &\textbf{1.196} &93  &\textbf{0.013} &1.368 &57 &0.024  &1.630   &25 &0.065   &1.598  &\textbf{24} &0.067   \\  \hline \rule{0pt}{7pt}
11  &1.480 &94  &\textbf{0.016} &\textbf{1.380} &58 &0.024  &1.760   &28 &0.063   &1.765  &\textbf{24} &0.074   \\  \hline \rule{0pt}{7pt}
Avg.&\textbf{0.670} &93  &\textbf{0.007} &0.679 &57 &0.012  &0.899   &24 &0.036 &0.917  &\textbf{23}  &0.039  \\
\hline
\end{tabular}
\end{center}
\end{table*}

\begin{table*}[t]
\begin{center}
 \caption{Computational complexity on EYaleB}
\centering
\label{EYaleBCC}
\begin{tabular}{|c|c|c|c|c|c|c|c|c|c|c|c|c|}
\hline &\multicolumn{3}{|c|}{SSC} & \multicolumn{3}{|c|}{LRR} & \multicolumn{3}{|c|}{MFC$_01$} & \multicolumn{3}{|c|}{MFC$_02$}\\
\hline
K  &T(s)    &I    &T/I    &T(s)  &I  &T/I     &T(s) &I  &T/I      &T(s) &I  &T/I    \\  \hline \rule{0pt}{6pt}
2 &1.73  &107   &\textbf{0.02}   &1.73  &61 &0.03    &1.14  &\textbf{24}  &0.05     &\textbf{0.96}  &\textbf{24}  &0.04    \\ \hline \rule{0pt}{7pt}
4 &5.69  &108   &\textbf{0.05}   &6.32   &64 &0.10    &2.70 &25 &0.11     &\textbf{2.19}  &\textbf{24}  &0.09    \\  \hline \rule{0pt}{7pt}
6 &12.20 &109   &\textbf{0.11}  &14.19   &64 &0.22    &4.95 &25 &0.20     &\textbf{3.36}  &\textbf{24}  &0.14    \\  \hline \rule{0pt}{7pt}
8 &21.33 &110   &0.19  &25.14  &64 &0.39    &7.62 &\textbf{27}  &0.28     &\textbf{5.48}  &32 &0.17    \\ \hline \rule{0pt}{7pt}
10  &37.54 &110   &0.34  &40.82  &64 &0.64   &11.17 &35 &0.32     &\textbf{7.75}  &\textbf{32}  &\textbf{0.24}    \\  \hline \rule{0pt}{7pt}
12  &52.08 &111   &0.47  &67.52  &65 &1.04   &14.32 &35 &0.41     &\textbf{9.77}  &\textbf{32}  &\textbf{0.31}    \\  \hline \rule{0pt}{7pt}
14  &69.28 &111   &0.62  &109.57 &66 &1.66   &19.09 &35 &0.54    &\textbf{12.39}  &\textbf{32} &\textbf{0.39} \\ \hline \rule{0pt}{7pt}
16  &93.08 &111   &0.84  &142.68 &65 &2.20   &22.96 &34 &0.67    &\textbf{16.18} &\textbf{32} &\textbf{0.51}   \\ \hline \rule{0pt}{7pt}
18  &124.00 &111  &1.12  &185.63 &65 &2.85   &28.12 &35 &0.80    &\textbf{19.17} &\textbf{32} &\textbf{0.60}   \\ \hline \rule{0pt}{7pt}
20  &152.89 &111  &1.38  &263.19 &72 &3.65   &34.02 &35 &0.97    &\textbf{23.60} &\textbf{32} &\textbf{0.74}   \\ \hline \rule{0pt}{7pt}
Avg.&56.98  &110  &0.51  &85.68  &65 &1.28   &14.61 &31 &0.44  &\textbf{10.09}  &\textbf{30}  &\textbf{0.32}   \\
\hline
\end{tabular}
\end{center}
\end{table*}

\begin{table*}[t]
\begin{center}
\caption{Computational complexity on USPS}
\centering
\label{USPSCC}
\begin{tabular}{|c|c|c|c|c|c|c|c|c|c|c|c|c|}
\hline &\multicolumn{3}{|c|}{SSC} & \multicolumn{3}{|c|}{LRR} & \multicolumn{3}{|c|}{MFC$_01$} & \multicolumn{3}{|c|}{MFC$_02$}\\
\hline
K  &T(s)    &I    &T/I    &T(s)  &I  &T/I     &T(s) &I  &T/I      &T(s) &I  &T/I    \\  \hline \rule{0pt}{6pt}
2 &6.28  &108   &0.06   &2.54  &61 &0.04    &1.22 &\textbf{27}  &0.05     &\textbf{0.91}  &31 &\textbf{0.03}    \\ \hline \rule{0pt}{7pt}
3 &13.21 &110   &0.12   &6.33  &63 &0.10    &1.87 &\textbf{27}  &0.07     &\textbf{1.38}  &31 &\textbf{0.04}    \\  \hline \rule{0pt}{7pt}
4 &22.07 &110   &0.20  &12.24  &63 &0.19    &2.47 &\textbf{29}  &0.09     &\textbf{1.94}  &30 &\textbf{0.06}    \\  \hline \rule{0pt}{7pt}
5 &29.18 &111   &0.26  &21.52  &63 &0.34    &3.13 &\textbf{30}  &0.10     &\textbf{2.66}  &\textbf{30} &\textbf{0.08}    \\ \hline \rule{0pt}{7pt}
6 &38.06 &111   &0.34  &34.84  &64 &0.54    &4.36 &40 &0.11     &\textbf{3.76}  &\textbf{36}  &\textbf{0.10}    \\  \hline \rule{0pt}{7pt}
7 &54.53 &111   &0.49  &107.60 &65 &1.65    &5.27 &39 &0.14     &\textbf{4.36}  &\textbf{36}  &\textbf{0.12}    \\  \hline \rule{0pt}{7pt}
8 &64.16 &111   &0.58  &137.31 &65 &2.11    &6.64 &40 &0.17     &\textbf{5.64}  &\textbf{36} &\textbf{0.16} \\ \hline \rule{0pt}{7pt}
9  &71.58  &111   &0.64  &192.85 &65 &2.97    &8.80 &39 &0.22     &\textbf{7.12} &\textbf{35} &\textbf{0.20}   \\ \hline \rule{0pt}{7pt}
10  &93.97 &111   &0.85  &255.04 &65 &3.92    &9.90 &39 &0.25     &\textbf{8.48} &\textbf{35} &\textbf{0.24}   \\ \hline \rule{0pt}{7pt}
Avg.&43.67 &110   &0.39  &85.59  &64 &1.32    &4.85 &34 &0.13   &\textbf{4.03}  &\textbf{33}  &\textbf{0.12}   \\
\hline
\end{tabular}
\end{center}
\end{table*}


\subsubsection{Face Basis Learning and Reconstruction on AR}

Subspace basis learning aims to learn the representative basis that provide a compact representation for each data sample. Traditional methods assume that the face images of all subjects can be represented by a single lower-dimensional subspace. However, previous work~\cite{basri2003lambertian} revealed that, under the Lambertian assumption, face images of a subject with a fixed pose and varying illumination lie close to a linear subspace. Thus for multiple subjects, the reasonable way is to learn each subject its underlying basis. To confirm this argument, we conduct an experiment to visualize the learnt face basis (Note that SSC and LRR are unable to learn the basis, thus we do not show their results). Moreover, we also display the reconstructed face image and the associated error. To see the difference with single subspace learning methods, we choose PCA for comparison.

\begin{figure}[!hbt]
\begin{center}
\subfigure[]{\label{fig8:a}\includegraphics[width=0.48\linewidth]{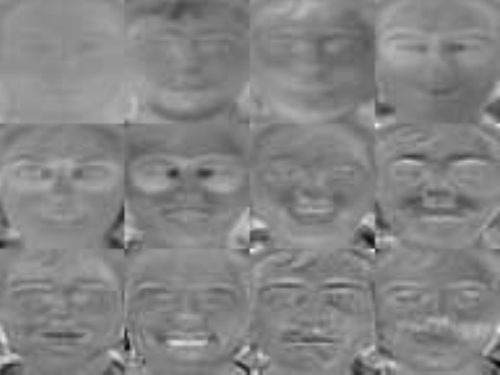}}
\subfigure[]{\label{fig8:b}\includegraphics[width=0.48\linewidth]{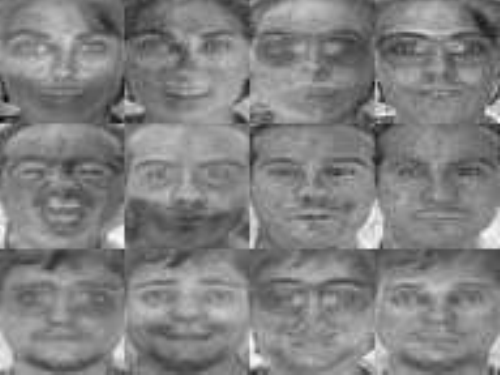}}
\subfigure[]{\label{fig8:c}\includegraphics[width=0.48\linewidth]{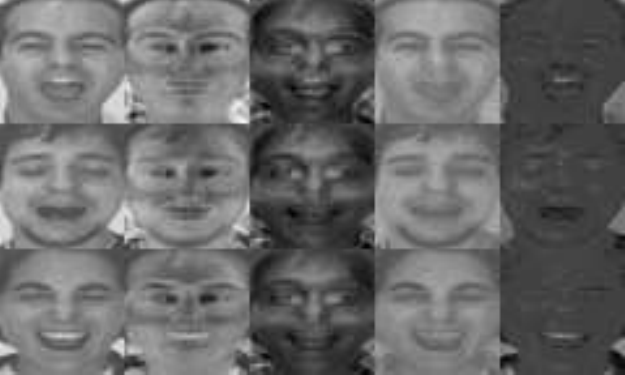}}
\subfigure[]{\label{fig8:d}\includegraphics[width=0.48\linewidth]{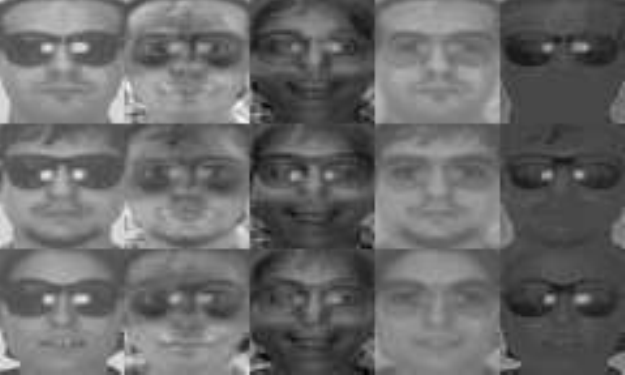}}
\end{center}
\caption{Some examples of learnt face basis, reconstructed face images, and errors of three subjects in AR. (a) and (b) are representative face basis obtained by PCA and MFC$_0$. The five columns in (c) and (d) are raw face images, reconstructed face images by PCA, errors corrected by PCA, reconstructed face images by MFC$_02$, and errors corrected by MFC$_02$, respectively. (c) shows results on some images with laughter, where ``teeth" is treated as the error. (d) show results on some images with wearing sunglass, where ``sunglass" is treated as the error. We do not show results on SSC and LRR as they are unable to learn the face basis.}
\label{fig8}
\vspace{-2mm}
\end{figure}


We randomly choose three subjects (subject 1, 2, and 55, each 15 samples) in AR to learn the basis, which are then used for face reconstruction. Since some face images in AR are grossly corrupted, for instance, wearing sunglass, we therefore treat them as sample-specific outliers and choose MFC$_0$2. Some examples of learnt face basis, reconstructed face images, and errors of PCA and MFC$_02$ are plotted in Figure $\ref{fig8}$. From it, we have several observations: 

(1) \textbf{MFC$_0$ can learn each subpace the corresponding basis}. The learnt basis via MFC$_02$ can be separated into three parts---basis in each row represent one subject, meaning face images of one subject lie on its own underlying subspace. In contrast, the basis of PCA are mixed together, i.e., each basis contains information of face images across the three subjects;

(2) \textbf{MFC$_0$ is effective in reconstructing face images}. The reconstructioned face images via MFC$_02$ are very clean, while they are blurry and untidy via PCA. 

(3) \textbf{MFC$_0$ is able to correct discriminant errors, which can further be used to perform recognition tasks}. The errors via MFC$_02$ are almost extracted. For Figure $\ref{fig8}$(c) and (d), the errors are  ``teeth"---One reasonable explanation is that face images in the training set do not show their teeth. Similarly, the errors in Figure $\ref{fig8}$(e) and (f) are ``sunglass". Interestingly, the errors contain discriminant features. In fact, they can be used for facial emotion recognition or object recognition~\cite{Liu:LLRR}. For instance, we can leverage these errors to detect whether one is laughing or wearing sunglass. Nevertheless, the errors obtained by PCA cannot provide any useful information.

\subsubsection{Summary} 
\begin{itemize}
\item MFC$_0$ (either MFC$_01$ or MFC$_02$) is not only effective for subspace clustering, subspace basis learning, but also implemented efficiently.

\item The advantage of MFC$_0$ against single subspace learning methods can be magnified when handling data generated from multiple subspaces. 

\item The advantage against SSC and LRR can be boosted by handling data contained gross errors or/and data with large sample size.
\end{itemize}

\section{Conclusions and Future Work}

In this paper, we propose a novel method, called Column $L_0$-norm constrained Matrix Factorization (MFC$_0$), for robustly analyzing the structure of data generated from multiple categories. By learning the basis with an orthonormal constraint, MFC$_0$ is able to discover the mixture subspace structure and is robust to different types of errors with the specific regularization. Moreover, MFC$_0$ directly imposes an $l_0$-norm constraint on the representation matrix, which achieves a (or approximate) block-diagonal structure when the data are clean (or contain errors). We propose a very efficient first-order alternating direction type algorithm to stably solve the nonconvex and nonsmooth objective function of MFC$_0$. Experimental results on synthetic data and real-world datasets verify that, besides the superiority and efficiency over traditional and state-of-the-art methods for multi-subspace recovery and clustering, MFC$_0$ also demonstrates its uniqueness for multi-subspace basis learning and direct representation learning.

Although MFC$_0$ owns these advantages, it also faces several problems that should be investigated in future. First, we assume that the multiple subspaces are independent, however in practice the data may be generated from disjoint subspaces. In this case, how can we incorporate the constraint or regularization into the model? More challenging, how can we process the data that are generated from nonlinear manifolds, instead of linear subspaces? Second, the computational complexity of MFC$_0$ is linear with respect to the sample size, so its extension to deal with massive dataset will be investigated in future.

\ifCLASSOPTIONcaptionsoff
  \newpage
\fi



%
\bibliographystyle{IEEEtran}
\bibliography{MFC0}

\appendix
\counterwithin{theorem}{section}

\subsection{Proof of Theorem~\ref{theorem_1}}
\label{app:theorem1} 
Let $\mathbf{Y}^{\star}$ be an optimal solution of Eq.($\ref{a10}$) under the given basis $\mathbf{X}$. We can then decompose $\mathbf{Y}^{\star}$ into two parts, i.e., the diagonal $\mathbf{Y}^{D}$ and non-diagonal $\mathbf{Y}^{N}$, where
\begin{equation*}
\mathbf{Y}^{D} =
\left[
     \begin{array}{cccc}
    \mathbf{\hat{Y}}_1^{\star} & \mathbf{0} & \cdots & \mathbf{0} \\
    \mathbf{0} & \mathbf{\hat{Y}}_2^{\star} & \cdots & \mathbf{0} \\
    \vdots   & \ddots & \vdots & \vdots \\
    \mathbf{0} & \mathbf{0}&  \cdots & \mathbf{\hat{Y}}_K^{\star} \\
  \end{array}
\right],
\mathbf{Y}^{N} =
\left[
     \begin{array}{cccc}
    \mathbf{0} & \star & \cdots & \star \\
    \star & \mathbf{0} & \cdots & \star \\
    \vdots   & \ddots & \vdots & \star \\
    \star &  \star & \cdots & \mathbf{0} \\
  \end{array}
\right]
\end{equation*}
with $\mathbf{\hat{Y}}_k^{\star} \in \Re^{d_0 \times n_k}$ and $\langle \mathbf{Y}^{D}, \mathbf{Y}^{N} \rangle = 0$.

Let $\mathbf{z}^{\star}_t, \mathbf{z}^{D}_t, \mathbf{z}^{N}_t$ respectively denote the $t$-th column of $\mathbf{X}\mathbf{Y}^{\star}, \mathbf{X}\mathbf{Y}^{D}, \mathbf{X}\mathbf{Y}^{N}$. Suppose that $\mathbf{z}^{\star}_t \in \mathcal{S}_k $. It is easy to check $\mathbf{z}^D_t \in \mathcal{S}_k$, therefore $\mathbf{z}^{N}_t \in \sum_{j \neq k} \mathcal{S}_j$. However, $\mathbf{z}^{N}_t = \mathbf{z}^{\star}_t - \mathbf{z}^{D}_t \in \mathcal{S}_k$. As all subspaces are independent, meaning $\mathcal{S}_k \cap \sum_{j \neq k} \mathcal{S}_j = \{ 0 \}$, we must require $\mathbf{z}^{N}_t = 0$. Thus $\mathbf{X}\mathbf{Y}^{N} = 0$, and $\mathbf{Z} = \mathbf{X}\mathbf{Y}^D$, $\mathbf{Y}^{\star}=\mathbf{Y}^D$. Therefore, the solution $\mathbf{Y}^{\star}$ to Eq.($\ref{a11}$) is block-diagonal.


\subsection{Proof of Theorem~\ref{theorem_2}}
\label{app:theorem2} 
We first unfold the objective function:
\begin{equation*}
\footnotesize
\min \left\| \mathbf{A}-\mathbf{DB} \right\|_F^2 = \textbf{tr}(\mathbf{A^TA}) - 2\textbf{tr}(\mathbf{AB^TD^T}) + \textbf{tr}(\mathbf{B^TD^TDB}).
\end{equation*}

Using the constraint $\mathbf{D^TD}=\mathbf{I}$ and the SVD of $\mathbf{AB^T}= \mathbf{L \Sigma R^T}$ transforms above minimization to the maximization problem
\begin{equation*}
\max \textbf{tr}(\mathbf{AB^TD^T}) = \textbf{tr}(\mathbf{L \Sigma R^T D^T}) = \textbf{tr}(\mathbf{\Sigma (DR)^T L}).
\end{equation*}
Note that $\left\| \mathbf{(DR)^T L} \right\|_F^2 = \textbf{tr}(\mathbf{(DR)^T L L^T (DR)}) = \textbf{tr}(\mathbf{I}_d)$ is a constant. As $\mathbf{\Sigma}$ is a diagonal matrix, so the maximum is achieved when $\mathbf{(DR)^T L}$ is also a diagonal matrix, with positive diagonal elements. Therefore, we have $\mathbf{DR} = \mathbf{L}$, and thus $\mathbf{D}=\mathbf{LR^T}$.

\subsection{Proof of Theorem~\ref{theorem_3}}
\label{app:theorem3}
For any $\mathbf{u}_i$, we first split its elements into the negative and positive parts, with corresponding indices defined as
$\mathcal{I}_+ = \left\{ j | u_{i,j} \geq 0\right\}$ and $\mathcal{I}_- = \left\{ j | u_{i,j} < 0 \right\}$. Denote $\| \mathbf{u}_i\|_+^2 = \sum_{j \in \mathcal{I}_+} u_{i,j}^2$ and $\| \mathbf{u}_i\|_-^2 = \sum_{j \in \mathcal{I}_-} u_{i,j}^2$. Then, we have two useful properties: 

(1) $\| \mathbf{u}_i \|_2^2 = \| \mathbf{u}_i \|_+^2 + \| \mathbf{u}_i \|_-^2$; 

(2) $\| \mathbf{v}_i - \mathbf{u}_i \|_+^2 + \| \mathbf{v}_i\|_-^2 = \| \mathbf{v}_i- \mathrm{P}_{+}(\mathbf{u}_i) \|_2^2$.

Property (1) is obvious and (2) can be simply proved: 
\begin{align*}
& \| \mathbf{v}_i- \mathrm{P}_{+}(\mathbf{u}_i) \|_2^2 = \| \mathbf{v}_i- \mathrm{P}_{+}(\mathbf{u}_i) \|_+^2 + \| \mathbf{v}_i- \mathrm{P}_{+}(\mathbf{u}_i) \|_-^2 \\
& \quad = \| \mathbf{v}_i- \mathbf{u}_i \|_+^2 + \| \mathbf{v}_i- 0 \|_-^2 = \| \mathbf{v}_i - \mathbf{u}_i \|_+^2 + \| \mathbf{v}_i\|_-^2.
\end{align*}

Based on above properties, we have
\begin{align*}
\label{prox_oper}
    \mathrm{prox}_{I_\mathcal{V}}(\mathbf{u}_i) &= \argmin_{\mathbf{v}_i} \left\{ || \mathbf{u}_i - \mathbf{v}_i ||_2^2: \mathbf{v}_i \geq 0, ||\mathbf{v}_i||_0 = d_0\right\} \\
    & = \argmin_{\mathbf{v}_i \in \mathcal{V}} \left\{ || \mathbf{u}_i - \mathbf{v}_i ||_2^2 \right\}  \\
    & = \argmin_{\mathbf{v}_i \in \mathcal{V}} \big\{ || \mathbf{u}_i - \mathbf{v}_i ||_+^2 + || \mathbf{u}_i - \mathbf{v}_i ||_-^2 \big\} \\
    & \Leftrightarrow \argmin_{\mathbf{v}_i \in \mathcal{V}} \left\{ || \mathbf{u}_i - \mathbf{v}_i ||_+^2 + || \mathbf{v}_i ||_-^2 - 2\sum_{j \in \mathcal{I}_-} v_{i,j} u_{i,j} \right\} \\
    & = \argmin_{\mathbf{v}_i \in \mathcal{V}} \left\{ || \mathbf{v}_i - \mathrm{P}_{+}(\mathbf{u}_i) ||_2^2 - 2\sum_{j \in \mathcal{I}_-} v_{i,j} u_{i,j} \right\}
\end{align*}
Note that we remove the constant $|| \mathbf{u}_i||_-^2 $ in the fourth equation. 

In the last equation, when $j \in \mathcal{I}_-$, $u_{i,j} <0$. As $v_{i,j} \geq 0$, so $-2\sum_{j \in \mathcal{I}_-} v_{i,j} u_{i,j} \geq 0$. Therefore, the minimum is achieved if and only if $v_{i,j} = 0$ for $j \in \mathcal{I}_-$, resulting in $||\mathbf{v}_i||_-^2 = 0$. Finally, we have 
\begin{equation*}
    \begin{split}
    \mathrm{prox}_{I_\mathcal{V}}(\mathbf{u}_i)
    & = \argmin_{\mathbf{v}_i} \left\{ || \mathbf{v}_i - \mathrm{P}_{+}(\mathbf{u}_i) ||_2^2: ||\mathbf{v}_i||_0 = d_0 \right\} \\
    & = \mathrm{P}_{d_0}(\mathrm{P}_{+}(\mathbf{u}_i)).
    \end{split}
\end{equation*}
\end{document}